\documentclass[letter,11pt]{article}

\newif\ifNeurIPS\NeurIPSfalse

\usepackage[english]{babel}
\usepackage[utf8]{inputenc} 
\usepackage[T1]{fontenc}    
\usepackage[colorlinks,allcolors=blue]{hyperref}

\usepackage{lmodern}
\usepackage{url}            
\usepackage{booktabs}       
\usepackage{amsfonts}       
\usepackage{nicefrac}       
\usepackage{microtype}      
\usepackage{enumitem}
\usepackage{dsfont}
\usepackage{subfig}
\usepackage{wrapfig, lipsum}
\usepackage{graphicx}
\usepackage[titletoc,title]{appendix}
\usepackage{amsmath,amsthm,amssymb,mathtools}
\usepackage{algorithm}
\usepackage{float}
\usepackage{sidecap}

\usepackage{braket}

\newtheorem{theorem}{Theorem}[section]
\newtheorem*{theorem*}{Theorem}

\newtheorem{proposition}[theorem]{Proposition}
\newtheorem*{proposition*}{Proposition}
\newtheorem{lemma}[theorem]{Lemma}
\newtheorem*{lemma*}{Lemma}

\newtheorem*{conjecture*}{Conjecture}
\newtheorem{fact}[theorem]{Fact}
\newtheorem*{fact*}{Fact}

\newtheorem*{hypothesis*}{Hypothesis}

\newtheorem*{claim*}{Claim}

\theoremstyle{definition}

\theoremstyle{remark}

\newtheorem*{remark*}{Remark}

\newtheorem*{observation*}{Observation}

\newcommand{\R}{\mathbb{R}}

\newcommand{\calX}{\mathcal{X}}
\newcommand{\calU}{\mathcal{U}}

\newcommand{\calY}{\mathcal{Y}}

\newcommand{\bfone}{\mathbf{1}}

\newcommand{\norm}[1]{\lVert #1 \rVert}

\newcommand{\Bignorm}[1]{\Big\lVert#1\Big\rVert}

\newcommand{\iprod}[1]{\langle#1\rangle}

\newcommand{\Esymb}{\mathbb{E}}
\newcommand{\Psymb}{\mathbb{P}}

\DeclareMathOperator*{\E}{\Esymb}

\DeclareMathOperator*{\ProbOp}{\Psymb}

\renewcommand{\Pr}{\ProbOp}

\newcommand{\tr}{\text{tr}}

\newcommand{\eps}{\varepsilon}
\renewcommand{\epsilon}{\varepsilon}

\newcommand{\sdpval}{\text{SDP}_{val}}

\newcommand{\diag}{\textrm{diag}}

\newcommand{\cub}{\text{val}_{\text{curr}}}

\newcommand{\miny}{y_{\text{ubmin}}}
\newcommand{\minval}{\text{upbd}_{\min}}
\newcommand{\accu}{\text{acc}}
\newcommand{\tPi}{{\Pi}}

\newif\ifnotes\notesfalse

\ifnotes
\usepackage{color}
\definecolor{mygrey}{gray}{0.50}
\newcommand{\notename}[2]{{\textcolor{mygrey}{\footnotesize{\bf (#1:} {#2}{\bf ) }}}}

\else

\newcommand{\notename}[2]{{}}

\fi

\newcommand{\pnote}[1]{{\notename{Pranjal}{#1}}}

\newcommand{\anote}[1]{{\notename{Aravindan}{#1}}}

\usepackage{tikz}
\usepackage{multirow}
\usepackage[noend]{algpseudocode}
\renewcommand{\algorithmicrequire}{\textbf{Input:}}

\usepackage{amsmath}
\usepackage{graphicx}
\usepackage[colorinlistoftodos]{todonotes}

\ifNeurIPS
\usepackage{neurips_2020}
\else
\usepackage{fullpage}
\fi

\title{Adversarial robustness via robust low rank representations}

\author{Pranjal Awasthi\thanks{Google Research and Department of Computer Science, Rutgers University. The author acknowledges support from the NSF HDR TRIPODS award CCF-1934924. {\tt pranjal.awasthi@rutgers.edu}.}
\and
 Himanshu Jain\thanks{Google Research} \and
 Ankit Singh Rawat\thanks{Google Research} \and
Aravindan Vijayaraghavan\thanks{
  Department of Computer Science,
  Northwestern University. Supported by the National Science Foundation (NSF) under Grant No.~CCF-1652491 and CCF-1637585. {\tt aravindv@northwestern.edu}. } 
}
\date{}

\begin{document}
\maketitle

\begin{abstract}
Adversarial robustness measures the susceptibility of a classifier to imperceptible perturbations made to the inputs at test time. In this work we highlight the benefits of natural low rank representations that often exist for real data such as images, for training neural networks with certified robustness guarantees.




Our first contribution is for certified robustness to perturbations measured in $\ell_2$ norm. We exploit low rank data representations to provide improved guarantees over state-of-the-art randomized smoothing-based approaches on standard benchmark datasets such as CIFAR-10 and CIFAR-100. 

Our second contribution is for the more challenging setting of certified robustness to perturbations measured in $\ell_\infty$ norm. We demonstrate empirically that natural low rank representations have inherent robustness properties, that can be leveraged to provide significantly better guarantees for certified robustness to $\ell_\infty$ perturbations in those representations. Our certificate of $\ell_\infty$ robustness relies on a natural quantity involving the $\infty \to 2$ matrix operator norm associated with the representation,  to translate robustness guarantees from $\ell_2$  to $\ell_\infty$ perturbations. 
%
A key technical ingredient for our certification guarantees is a fast algorithm with provable guarantees based on the multiplicative weights update method to provide upper bounds on the above matrix norm. Our algorithmic guarantees improve upon the state of the art for this problem, and may be of independent interest. 

\end{abstract}

\section{Introduction}\label{sec:intro}
It is now well established across several domains like images, audio and natural language, that small input perturbations that are imperceptible to humans can fool deep neural networks at test time~\cite{szegedy2013intriguing, biggio2013evasion,carlini2018audio, ebrahimi2017hotflip}. 
This phenomenon known as {\em adversarial robustness} has led to flurry of research in recent years~(see Section~\ref{app:related} for a discussion of related work). Following most prior work in this area~\cite{goodfellow2014explaining, madry2017towards, zhang2019theoretically, shafahi2019adversarial, wong2020fast}, we will study the setting where adversarial perturbations to an input $x$ are measured in an $\ell_p$ norm~($p=2$ or $p=\infty$). 

In this work, we study methods for {\em certified adversarial robustness} in the framework developed in~\cite{lecuyer2019certified, cohen2019certified}. The goal is to output a classifier $f$ that on input $x \in \R^n$ outputs a prediction $y$ in the label space $\calY$,  along with a certified radius $r_f(x)$. The classifier is guaranteed to be robust at $x$ up to the radius $r_f(x)$ (with high probability), i.e., $\forall z: \|z\|_p \leq r_f(x),  f(x+z)=f(x)$. For an $\ell_p$ norm and $\epsilon > 0$, the certified accuracy of a classifier $f$ is defined as
\begin{align}
    \label{eq:def-certified-acc}
    {\accu}^{(\ell_p)}_\epsilon(f) &= \Pr_{(x,y) \sim D} \big[f(x)=y \text{ and } r_f(x) \geq \epsilon \big],
\end{align}
where $D$ is the underlying data distribution generating test inputs. We call the radius $r_f(x)$ returned by the classifier as the {\em certified radius} on $x$. When $\epsilon = 0$ this is the natural accuracy of $f$.

For certified adversarial robustness to $\ell_2$ perturbations, the {\em randomized smoothing} procedure proposed in~\cite{lecuyer2019certified, cohen2019certified} is a simple and efficient method that can be applied to {\em any} neural network. 
Randomized smoothing works by creating a smoothed version of a given classifier by adding Gaussian noise to the inputs (see Section~\ref{sec:l2}). The smoothed classifier exhibits certain Lipschitzness properties, and one can derive good certified robustness guarantees from it. The study of randomized smoothing for certified $\ell_2$ robustness is an active research area and the current best guarantees are obtained by incorporating the smoothed classifier into the training process~\cite{salman2019provably} (see Section~\ref{app:related}).

It seems much more challenging to obtain certified adversarial robustness to $\ell_\infty$ perturbations~\cite{raghunathan2018semidefinite, gowal2018effectiveness, wong2018provable}. In particular, the design of a procedure akin to randomized smoothing has been difficult to achieve for $\ell_\infty$ perturbations. 
\anote{mention non basis independence of $\ell_\infty$ somewhere else}
One approach to obtain certified $\ell_\infty$ robustness is to translate a certified radius guarantee of $\epsilon$ for $\ell_2$ perturbations~(via randomized smoothing) into an $\epsilon/\sqrt{n}$ certified radius guarantee for $\ell_\infty$ perturbations; here $n$ is the dimensionality of the ambient space. 
Furthermore, recent work~\cite{blum2020random, yang2020randomized, kumar2020curse} has established lower bounds proving that randomized smoothing based methods cannot break the above $\sqrt{n}$ barrier for $\ell_\infty$ robustness in the worst case.

However real data such as images are not worst case and often exhibits a natural low rank structure. In this work we show how we can leverage such natural low-rank representations for the data, in order to design algorithms based on randomized smoothing with improved certified robustness guarantees for both $\ell_2$ and $\ell_\infty$ perturbations. 

\vspace{-5pt}

\paragraph{Our Contributions.} We now describe our main contributions.

\vspace{4pt}
\noindent {\em Improved certified $\ell_2$ robustness: } Our first contribution is to design new smoothed classifiers for achieving certified robustness to $\ell_2$ perturbations. These classifiers achieve improved tradeoffs between natural accuracy and certified accuracy at higher radii.
We achieve this by leveraging the existence of good low-rank representation for the data. 
%
We modify the randomized smoothing approach to instead selectively inject more noise along certain directions, without compromising the accuracy of the classifier. 
The large amount of noise leads to classifier that is less sensitive to $\ell_2$ perturbations, and hence achieves higher certified accuracy across a wide range of radii. We empirically demonstrate the improvements obtained by our approach on image data in Section~\ref{sec:l2}. 

\vspace{4pt}
\noindent {\em Fast algorithms for translating certified robustness guarantees from $\ell_2$ to $\ell_\infty$:} 
For the more challenging setting of $\ell_\infty$ robustness we consider classifiers of the form $f(Px)$ where $P$ is an arbitrary linear map, and $f$ represents an arbitrary neural network. 
When translating certified robustness guarantees for $\ell_2$ perturbations to obtain guarantees for $\ell_\infty$ perturbations, the loss incurred is captured by the $\infty \to 2$ operator norm of matrix $P$. 
While computing this operator norm is NP-hard, we design a fast approximate algorithm based on the multiplicative weights update method with provable guarantees. Our algorithmic guarantees give significant improvements over the best known bounds~\cite{AroraHK, Kalethesis} for this problem and may be of independent interest (see Section~\ref{sec:mw}). 

\vspace{4pt}
\noindent {\em Certified $\ell_\infty$ robustness in natural data representations:} Real data such as images have natural representations that are often used in image processing e.g., via the Discrete Cosine Transform (DCT). Via an empirical study we highlight the need for achieving $\ell_\infty$ robustness in the DCT basis. 
More importantly, we demonstrate that the representation in the DCT basis is robust, i.e., there exist low rank projections that capture most of the signal in the data and that at the same time have small $\infty \to 2$ operator norm.\footnote{This is also true for domains such as audio in the DCT basis. See Appendix~\ref{app:audio} for experimental evidence.}
We develop a fast heuristic based on sparse PCA to find such robust projections. When combined with our multiplicative weights based algorithm, this leads to a new training procedure based on randomized smoothing. Our procedure can be applied to any network architecture and
provides stronger guarantees on robustness to $\ell_\infty$ perturbations in the DCT basis. 


\section{Certified Robustness to $\ell_2$ Perturbations}
\label{sec:l2}
We build upon the {\em randomized smoothing} technique proposed in \cite{lecuyer2019certified, cohen2019certified} and further developed in \cite{salman2019provably}. Consider a multiclass classification problem and a classifier $f: \R^n \to \mathcal{Y}$, where $\mathcal{Y}$ is the label set.
Given $f$, randomized smoothing produces a smoothed classifier $g$ where
\begin{align}
    \label{eq:smoothed-classifier}
    g(x) = \arg \max_{y \in \mathcal{Y}} \mathbb{P}(f(x+\delta) = y).
\end{align}
Here $\delta \sim N(0, \sigma^2 I)$ is the Gaussian noise added. The following proposition holds. 
\begin{proposition}[\cite{lecuyer2019certified, cohen2019certified}]
Given a classifier $f$, let $g$ be its smoothed version as defined in \eqref{eq:smoothed-classifier} above. On an input $x$, and for $\delta \sim N(0,\sigma^2 I)$ define $y_A \coloneqq \arg \max_y \mathbb{P}(f(x+\delta) = y)$, and let $p_A = \mathbb{P}(f(x+\delta) = y_A)$. Then the prediction of $g$ at $x$ is unchanged up to $\ell_2$ perturbations of radius \begin{align}
\label{eq:certified-radius}
r(x) = \frac{\sigma}{2} \big(\Phi^{-1}(p_A) - \Phi^{-1}(p_B)\big).
\end{align}
Here $p_B = \max_{y \neq y_A} \mathbb{P}(f(x+\delta) = y)$ and $\Phi^{-1}$ is the inverse of the standard Gaussian CDF.
\end{proposition}
Hence, randomized smoothing provides a fast method to certify the robustness of any given classifier on various inputs. 
In order to get robustness to large perturbations it is desirable to choose the noise magnitude $\sigma$ as large as possible. However, there is a natural tradeoff between the amount of noise added and the natural accuracy of the classifier.
As an example consider an input $x \in \R^n$ of $\ell_2$ length $\sqrt{n}$. If $\sigma$ is the average amount of noise added then one is restricted to choosing $\sigma$ to be a small constant in order for the noise to not overwhelm the signal. 

However, it is well known that natural data such as images are low dimensional in nature. Figure~\ref{fig:pca_error_vs_d} in Appendix~\ref{sec:app-l2} shows that for the CIFAR-10 and CIFAR-100 datasets, even when projected, via  PCA, onto $200$ dimensions, the reconstruction error remains small.
If the input is close to an $r$-dimensional subspace, then it is natural to add noise only within the subspace for smoothing. Formally, let $\Pi$ be the projection matrix on to an $r$-dimensional subspace and $x$ be such that $\|\Pi x\|_2 \approx \|x\|_2 = \sqrt{n}$. For $\delta \sim N(0,\sigma^2 I)$ we have $\|\Pi \delta\|_2 \approx \sigma \sqrt{r}$. Hence if we only add noise within the subspace, then $\sigma$ can be as large as $\sqrt{n/r}$ as opposed to a constant without significantly affecting the natural accuracy.

We formalize this into an efficient training algorithm as follows: we take a base classifier/neural network $f(x)$ and replace it with the smoothed classifier $g_{\Pi}(x)$ where
\begin{align}
    \label{eq:smoothed-projected-classifier}
    g_\Pi(x) = \arg \max_{y \in \mathcal{Y}} \mathbb{P}(f(\Pi x+\delta_\Pi) = y).
    \end{align}
where $\Pi$ is a projection matrix onto an $r$-dimensional subspace and $\delta_\Pi$ is a standard Gaussian of variance $\sigma^2$ that lies within $\Pi$. For data such as images, good projections $\Pi$ can be obtained via methods like PCA. Furthermore, certifying the robustness of our proposed smoothed classifier can be easily incorporated into existing pipelines for adversarial training with minimal overhead. In particular using the rotational symmetry of Gaussian distributions it is easy to show the following
\begin{proposition}
\label{prop:l2-equivalence}
Given a base classifier $f: \R^n \to \mathcal{Y}$ and a projection matrix $\Pi$, on any input $x$, the smoothed classifier $g_\Pi(x)$ as defined in \eqref{eq:smoothed-projected-classifier} is equivalent to the classifier given by
\begin{align}
    \label{eq:smoothed-projected-classifier-v2}
    \tilde{g}_\Pi(x) = \arg \max_{y \in \mathcal{Y}} \mathbb{P}(f \big(\Pi (x+\delta) \big) = y).
    \end{align}
    Here $\delta$ is a standard Gaussian of variance $\sigma^2$ in every direction. 
\end{proposition}

Hence constructing our proposed smoothed classifier simply requires adding a linear transformation layer to any existing network architecture before training via randomized smoothing. We propose to train the smoothed classifier as defined in \eqref{eq:smoothed-projected-classifier-v2} by minimizing its adversarial standard cross entropy loss as proposed in~\cite{madry2017towards}. However, since dealing with $\arg \max$ is hard from an optimization point of view, we follow the approach of \cite{salman2019provably} and instead minimize the cross entropy loss of the following soft classifier
\begin{align}
    \label{eq:soft-classifier}
    G_\Pi(x) = \E_{\delta \sim N(0,\sigma^2 I)}[f\big(\Pi (x+\delta) \big)].
\end{align}
This leads to the following objective where $\ell_{ce}$ is the standard cross-entropy objective and $\epsilon > 0$ is perturbation radius chosen for the training procedure.
\begin{align}
\label{eq:pgd-objective}
    \arg \min_f \E_{(x,y)} \big[\max_{z: \|z\|_2 \leq \epsilon} \ell_{ce}(G_\Pi(x+z),y) \big].
\end{align}
Following \cite{madry2017towards,  salman2019provably}, the inner maximization problem of finding adversarial perturbations is solved via projected gradient descent~(PGD), and given the adversarial perturbations, the outer minimization uses stochastic gradient descent. 
Overall, this leads to the following training procedure.
\begin{algorithm}[H]
\caption{Adversarial training via projections}
\label{algo:l2-training}
\begin{algorithmic}[1]
\Function{RobustTrain}{training data $(x_1, y_1), \dots, (x_m, y_m)$, subspace rank $r$, base noise magnitude $\sigma$, $\lambda \in [0,1]$, number of steps $T$, mini batch size $b$} 
\State Perform PCA on (unlabeled) data matrix $A \in \R^{n \times m}$ to obtain a rank-$r$ projection matrix $\Pi$. 
\State Set $G_\Pi$ as in \eqref{eq:soft-classifier} with $\sigma = \lambda \sqrt{n/r}$. 
\For{$t = 1,\dots, T$}
\State Obtain a mini batch of $b$ examples $(x_{t_1}, y_{t_1}), \dots, (x_{t_b}, y_{t_b})$. 
\State For each $x_{t_i}$ use projected gradient ascent on inner maximization in \eqref{eq:pgd-objective} to get $x'_{t_i}$.
\State Given perturbed examples $\{(x'_{t_i}, y_{t_i})\}_{i \in [b]}$, update network parameters via SGD.
\EndFor
\State Output the smoothed classifier $\tilde{g}_\Pi(x)$.
\EndFunction

\end{algorithmic}
\end{algorithm}
\vspace{-10pt}


\noindent \textbf{Empirical Evaluation.} We compare Algorithm~\ref{algo:l2-training} with the algorithm of \cite{salman2019provably} for various values of $\sigma$ and $\epsilon$ (used for training to optimize \eqref{eq:pgd-objective}). We choose $\epsilon \in \{0.25, 0.5, 0.75, 1.0\}$ and for each $\epsilon$ we choose the value of $\sigma$ as described in~\cite{salman2019provably}. In each case, we train the classifier proposed in \cite{salman2019provably} using a noise magnitude $\sigma$, and we train our proposed smoothed classifier using higher noise values of $\lambda \sigma \sqrt{n/r}$, where $\lambda$ is a parameter that we vary. In all experiments, we train a ResNet-32 network on the CIFAR-10 dataset by optimizing \eqref{eq:pgd-objective}. Figure \ref{fig:acc_tradeoff_cifar10} shows a comparison of certified accuracies for different radii and different values of $\lambda$.  See Appendix~\ref{sec:app-l2} for a description of the hyperparameters and additional experiments. 

\vspace{-8pt}
\begin{SCfigure}[][htbp]
    \centering
   
   \caption{\label{fig:acc_tradeoff_cifar10}Plot of certified accuracy achieved at different radii when different values of $\lambda$ are chosen to optimize the smoothed classifier in \eqref{eq:smoothed-projected-classifier-v2}. We compare with the method of \cite{salman2019provably} The plot is obtained by training a ResNet-32 architecture on CIFAR-10 with $\epsilon=1.0$. Note that the $y$-intercept of each curve represents the natural accuracy of the corresponding classifier.} 
        \includegraphics[width=0.56\textwidth]{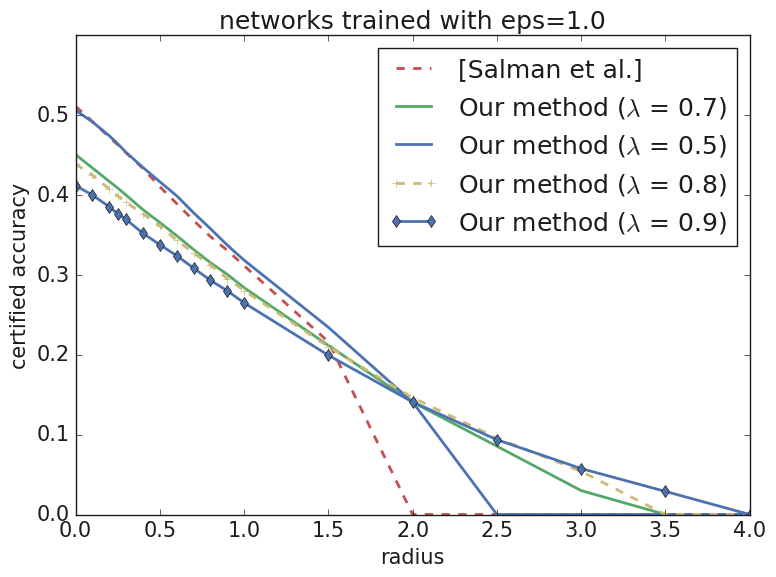} 

\end{SCfigure}
\vspace{-8pt}

As can be seen from the figure, varying the value of $\lambda$ lets us tradeoff lower accuracy at small values of the radius for a significant gain in certified accuracy at higher radii as compared to the method of \cite{salman2019provably}. In particular we find that choosing values of $\lambda$ close to $0.5$ leads to networks that can certify accuracy at much higher radii with minimal to no loss in the natural accuracy as compared to the approach of \cite{salman2019provably}. In Figure \ref{fig:l2} we present the result of our training procedure for various values of $\epsilon$ and $\sigma$ and compare with the $\ell_2$ smoothing method of \cite{salman2019provably} on the CIFAR-10 and CIFAR-100 datasets. In each case, to obtain the projection matrix $\Pi$ we perform a PCA onto each image channel separately and use the top $200$ principal components to obtain the projection matrix $\Pi$. 
For both datasets, our trained networks outperform the method of~\cite{salman2019provably} across a large range of radius values. In particular, for higher values of radius (say, $\gtrapprox 0.5$) our method achieves a desired certified accuracy with significantly higher natural accuracy as compared to the method of~\cite{salman2019provably}. For instance in the CIFAR-10 dataset, at a radius of $1.0$ and a desired certified accuracy of at least $0.35$, the method of~\cite{salman2019provably} achieves a natural accuracy of $\approx 0.5$ (yellow dotted curve at radius $0$). In contrast our method achieves the same with a natural accuracy of $\approx 0.65$ (green solid curve at radius $0$). On the other hand, at very small radius values the method of~\cite{salman2019provably} is better. This is expected as we suffer a small loss in natural accuracy due to the PCA step in Algorithm~\ref{algo:l2-training}. 

\begin{figure}[t]
\centering
\subfloat{\includegraphics[width=6.5cm]{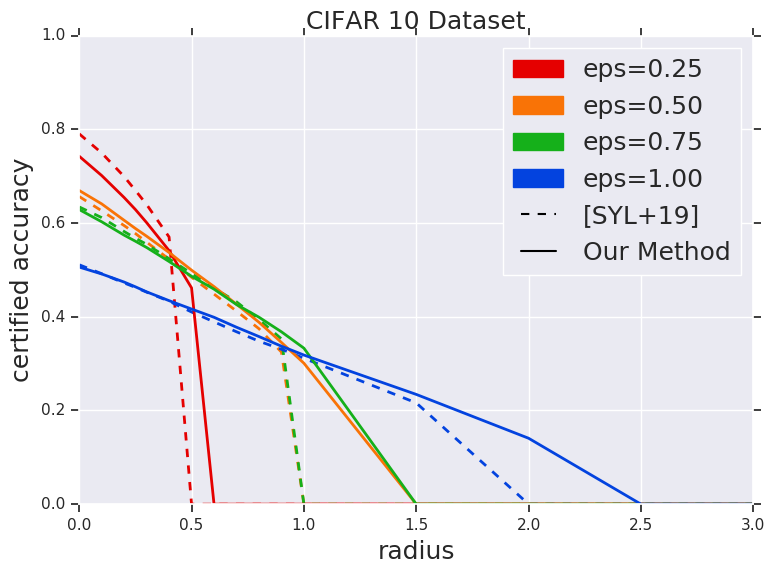}}\hfil
\subfloat{\includegraphics[width=6.5cm]{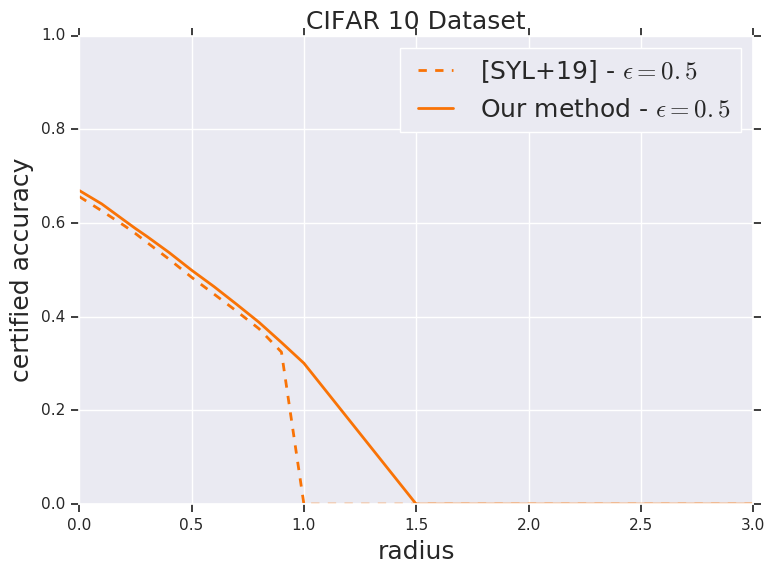}}\hfil

\subfloat{\includegraphics[width=6.5cm]{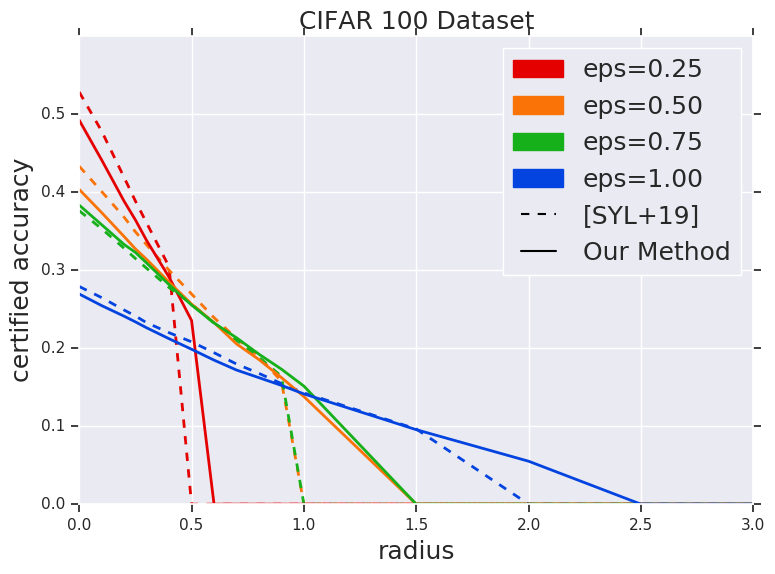}}\hfil
\subfloat{\includegraphics[width=6.5cm]{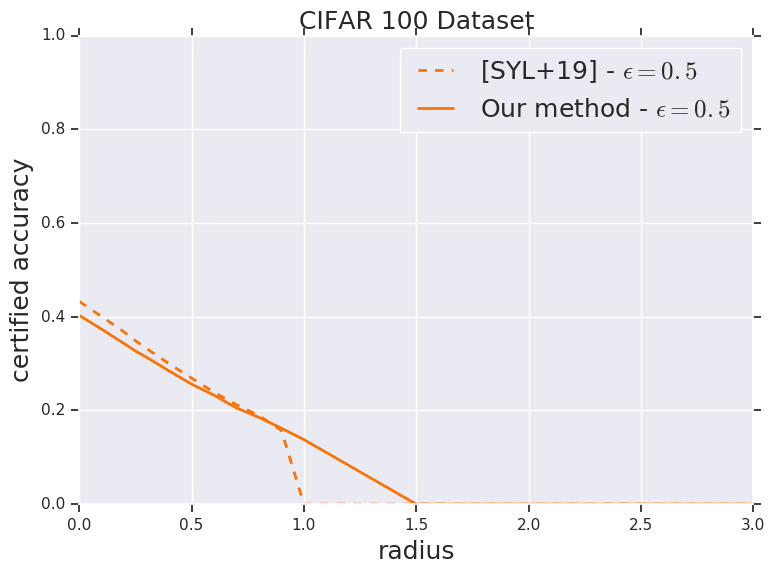}}\hfil
\caption{\label{fig:l2} A comparison of certified radius guarantees obtained via Algorithm~\ref{algo:l2-training} as compared to the approach of~\cite{salman2019provably}. The x-axis is the radius, and the y-axis represents the certified accuracy. The top row describe results for the CIFAR-10 dataset -- (left) certified accuracies for various values of $\epsilon$, (right) for $\epsilon=0.5$. Similarly, the bottom row describe the results for the CIFAR-100 dataset.}
\vskip -.15in
\end{figure}

\vspace{-10pt}

\section{Methods for Certified $\ell_\infty$ Robustness} \label{sec:linf:methods}


We now describe our algorithms for the more challenging problem of certified robustness to $\ell_\infty$ perturbations in a given basis or representation. 
Our approach is to leverage the existence of good representations of natural data measured by a certain robustness parameter, and translate $\ell_2$ robustness guarantees from Section~\ref{sec:l2} to get certified $\ell_\infty$ guarantees. 
Consider $f: \R^n \to \calY$ and $g(x):=f(\Pi x)$, where $\Pi \in \R^{n \times n}$ represents a projection matrix. 
The certified accuracy of $g$ satisfies
\begin{equation}\label{eq:linftol2}
\forall \eps>0, ~~ \accu^{(\ell_\infty)}_{\eps'}(g) \ge \accu^{(\ell_2)}_{\eps}(g) \text{ for } 0 \le \eps' \le \eps/ \norm{\Pi}_{\infty \to 2}, ~~\text{where } \norm{\Pi}_{\infty \to 2}= \max_{x:
\norm{x}_\infty \le 1} \norm{\Pi x}_2  \nonumber
\end{equation}
is the $\infty \to 2$ operator norm of $\Pi$ and represents a robustness parameter
(see Proposition~\ref{prop:linftol2} for a formal claim).
Hence, to translate guarantees from $\ell_2$ to $\ell_\infty$, we look for {\em robust} projections $\Pi$ that have small $\infty \to 2$ operator norm.
Our approach is inspired by the recent theoretical work of ~\cite{ACCV}, and finds a low-dimensional representation given by an (orthogonal) projection matrix $\Pi$ for the data that has small $\norm{\Pi}_{\infty \to 2}$. 
By matrix norm duality, there is a nice characterization for $\norm{\Pi}_{\infty \to 2}$ as the maximum $\ell_1$ norm among Euclidean unit vectors in the subspace of $\Pi$ (this is a notion of sparsity of the vectors). 
For a rank $r$ projector $\Pi$, the range of values taken by $\norm{\Pi}_{\infty \to 2}$ is $[\sqrt{r},\sqrt{n}]$. Hence if the projection is not low-rank and sparse, $\norm{\Pi}_{\infty \to 2}$ could be as large as $\sqrt{n}$ (e.g., when $\Pi=I$). This is consistent with the loss of $\sqrt{n}$ factor in robustness radius to $\ell_\infty$ perturbations for general datasets~\cite{blum2020random, yang2020randomized}. Moreover as we have seen in Section~\ref{sec:l2}, good low-rank representations of the data also give stronger certified $\ell_2$ robustness guarantees (and in turn, stronger certified $\ell_\infty$ guarantees). 

Our goal is to find a good robust rank-$r$ projection of the data if it exists. 
We propose a heuristic based on sparse PCA~\cite{mairal2009sparsePCA} to find a robust low rank projection with low reconstruction error (see Section~\ref{sec:robustprojheuristic}).
Since we aim for certified robustness, an important step in this approach is to compute and certify an upper bound on  
$\norm{\Pi}_{\infty \to 2}$, for our candidate projection $\Pi$. This is an NP-hard problem, related to computing the famous Grothendieck norm~\cite{alon2004approximating}. We describe a new, scalable, approximate algorithm for computing upper bounds on $\norm{\Pi}_{\infty \to 2}$ with provable guarantees.


\ifNeurIPS
\paragraph{Certifying the $\infty \to 2$ operator norm.} \label{sec:mw}
\else
\subsection{Certifying the $\infty \to 2$ operator norm.} \label{sec:mw}

Our fast algorithm is based on the multiplicative weights update (MWU) method for approximately solving a natural semi-definite programming (SDP) relaxation, and produce a good upper bound on $\norm{\Pi}_{\infty \to 2}$. 
Our upper bound also comes with a certificate from the dual SDP i.e., a short proof of the correctness of the upper bound.  
Given a candidate $\Pi$ our algorithm will compute an upper bound for $\norm{\Pi}_{\infty \to 1}$ which by matrix norm duality satisfies 
\begin{equation}\label{eq:normduality}
    \| \Pi\|_{\infty \to 2}^2 = \|\Pi\|_{\infty \to 1} = \max_{x}x^\top \Pi x~\text{ subject to } \norm{x}_\infty \le 1. 
\end{equation}
In fact our algorithm will work for the following more general problem of Quadratic Programming: 
\begin{equation} \label{eq:QP:problem}
\text{Given a symmetric matrix }M~\text{with } ~\forall i \in [n]:~ M_{ii} \ge 0 ,~\quad \max_{x: \|{x}\|_\infty \le 1}x^\top M x.
\end{equation}

%
 The standard SDP relaxation for the problem (see \eqref{eq:primal} in Appendix~\ref{sec:mw-app}), has primal variables represented by the positive semi-definite (PSD) matrix $X \in \R^{n \times n}$ satisfying constraints $X_{ii} \le 1$ for each $i \in [n]$. The SDP dual of this relaxation (given in \eqref{eq:dual} of Appendix~\ref{sec:mw-app}) has variables $y_1, \dots, y_n \ge 0$ corresponding to the $n$ constraints in the primal SDP. 
Since the SDP is a valid relaxation for \eqref{eq:QP:problem}, it provides an upper bound for $\infty \to 1$ operator norm\footnote{A fast algorithm that potentially finds a local optimum for the problem will not suffice for our purposes; we need an upper bound on the global optimum.} . Classical 
results show that it is always within a factor of $\pi/2$ of the actual value of $\norm{M}_{\infty \to 1}$~\cite{nesterov1998semidefinite,alon2006quadratic}. 
However, it is computationally intensive to solve the SDP using off-the-shelf SDP solvers (even for CIFAR-10 images, $X$ is $1024 \times 1024$). We design a fast algorithm based on the {\em multiplicative weight update} (MWU) framework~\cite{kleinLu,AroraHK}.  

%


\paragraph{Description of the algorithm}

Our algorithm differs slightly from the standard MWU approach for solving the above SDP. The algorithm below takes as input a matrix $M$ and always returns a valid upper bound $\minval$ on the SDP value, along with a dual feasible solution $\miny$ that can act as a certificate, 
and a candidate primal solution $X$ that attains the same value (and is potentially feasible). Theorem~\ref{thm:sdp-combined} proves that for the right setting of parameters, particularly the number of iterations $T_f = O(n \log n/ \delta^3)$, the solution $X$ is also guaranteed to be feasible up to small error $\delta>0$.

\begin{algorithm}[H]
\caption{Fast Certification of $\infty \to 1$ norm and Quadratic Programming}
\label{algo:sdp}
\begin{algorithmic}[1]
\Function{CertifySDP}{$M \in \R^{n \times n}$, iteration bound $T_f$, slack $\delta$, damping $\rho$}
\State Initialize $\alpha=(1,1,\dots,1) \in \R^n$. primal $X=0$, dual $y=0^n$, $\minval=\infty$, $\miny=0^n$.
\For{$t=0,1 \dots,  T$} 
    \State $\tilde{\alpha} \gets (1-\delta)\alpha+\delta (1,1,\dots, 1)$. 
    \State $\lambda \gets \text{max-eigenvalue}(\diag(\tilde{\alpha})^{-1/2} M \diag(\tilde{\alpha})^{-1/2})$ and $u \in \R^n$ be its eigenvector. 
    \State $v \gets \sqrt{n} \cdot \diag(\tilde{\alpha})^{-1/2} u $, 
    $y \gets \tfrac{1}{t+1} (t y + \lambda \tilde{\alpha})$, $X \gets \tfrac{1}{t+1}(t X + vv^\top)$. \label{eq:iterate:def}
     \If {$\norm{v}_\infty \le 1+\delta$ or $\max_i X_{ii} \le 1+\delta$ }, do early stop and return appropriate values. 
     \EndIf
    \State Update $\forall i\in [n], \alpha(i) \gets \alpha(i) \exp\big(\frac{\delta}{\rho}(v(i)^2 -1)\big)$, and renormalize s.t. $\sum_{i=1}^n \alpha(i)=n$.
    \State {\bf if} $\minval> n \lambda$, {\bf then} set $\minval=n \lambda$ and $\miny=y$.
\EndFor   
\State Output $\minval$, dual solution $\miny$, and primal candidate $X$. 

\EndFunction
\end{algorithmic}
\end{algorithm}

 Recall that from \eqref{eq:normduality} an estimate of the the $\infty \to 1$ norm immediately translates to an estimate of the $\infty \to 2$ norm.   
In the above algorithm, there are weights given by $\alpha$ for $n$ different constraints of the form $X_{ii} \le 1$. At each iteration, the algorithm maximizes the objective subject to {\em one} constraint of the form $\sum_i \tilde{\alpha}_i X_{ii} \le n$, where $\tilde{\alpha}$ involves a small correction to $\alpha$ that is crucial to ensure the run-time guarantees. The maximization is done using a maximum eigenvalue/eigenvector computation. The weights $\alpha$ are then updated using a multiplicative update based on the violation of the solution found in the current iterate.  
The damping factor $\rho$ determines the rate of progress of the iterations -- the smaller the value of $\rho$ the faster the progress, but a very small value may lead to oscillations. 
A more aggressive choice of $\rho$ compared  to the one in Theorem~\ref{thm:sdp-combined} seems to work well in practice. Finally, we remark that for {\em every} choice of $\alpha$ and $\rho$ we get a valid upper bound (due to dual feasibility). 
We show the following guarantee for our algorithm for the more general problem of \eqref{eq:QP:problem}.
%



\begin{theorem}\label{thm:sdp-combined}
Suppose $\delta>0$, and $M$ be any symmetric matrix with $M_{ii} \ge 0~\forall i \in [n]$. 
For any $\alpha \in \R_{\ge 0}^n $ with $\sum_{i=1}^n \alpha(i) =n$, if $\lambda = \lambda_{\max}\Big( (\diag(\alpha)^{-1/2} M \diag(\alpha)^{-1/2}) \Big)$, then $y=\lambda \alpha$ is feasible for the dual SDP and gives a valid upper bound of $n \lambda$ on the objective value for the SDP relaxation to \eqref{eq:QP:problem}. 
Moreover Algorithm~\ref{algo:sdp} on input $M$, with parameters $\delta$ and $\rho=O(n/\delta)$ after $T=O(n \log n/\delta^3)$ iterations finds a feasible SDP solution $\widehat{X} \succeq 0$ and a feasible dual solution $\widehat{y} \in \R^n$ that both sandwich the optimal SDP value within a $1+\delta$ factor. 
\end{theorem}
(See Prop~\ref{prop:dual} and Theorem~\ref{thm:sdp-guarantee} in Appendix~\ref{sec:mw-app} for formal statements along with proofs. ) 
Each iteration only involves a single maximum eigenvalue computation, which can be done up to $(1+\delta)$ accuracy in $T_{eig}=\tilde{O}(m/\delta)$ time where $m$ is the number of non-zeros in $M$ (see e.g.,~\cite{kleinLu}). 
To the best of our knowledge, this gives significant improvements over the prior best bound of $\tilde{O}(n^{1.5} m/\delta^{2.5})$ runtime for solving the above SDP~\cite{AroraHK,Kalethesis}. Our algorithm and analysis differs from the general MWU framework~\cite{Kalethesis} by treating the objective differently from the constraints so that the ``width parameter'' does not depend on the objective. A crucial step in our algorithm and proof is to add a correction term of $O(\delta)$ to the weights in each step that ensures that the potential violation of each constraint in an iteration is bounded.
In addition our algorithm is more scalable than existing off-the-shelf methods. See Appendix~\ref{sec:mw-app} and \ref{app:sdp-runtme} for details and comparisons.  

\ifNeurIPS

\else
\subsection{Finding Robust Low-rank Representations} \label{sec:robustprojheuristic}

We now show how to find a good low-rank robust projections when it exists in the given representation. Given a dataset $A$, our goal is to find a (low-rank) projection $\Pi$ that gets small reconstruction error $\norm{A - \Pi A}_F^2$, while ensuring that $\norm{\Pi}_{\infty \to 2}$ is small. Awasthi et al.~\cite{ACCV} formulate this as an optimization problem that is NP-hard, but show polynomial time algorithms based on the Ellipsoid algorithm that gives constant factor approximations. However, the algorithm is impractical in practice because of the Ellipsoid algorithm, and the separation oracle used by it (that in turn involves solving an SDP). 
We instead use the connection to sparsity described in Section~\ref{sec:linf:methods} to devise a much faster heuristic based on sparse PCA to find a good projection $\Pi$ (see \eqref{fact:sparsityconnection} in the appendix for a formal justification). We just use an off-the-shelf heuristic for sparse PCA (the scikit-learn sparse PCA implementation~\cite{scikit-learn} based on ~\cite{mairal2009sparsePCA}), along with our certification procedure Algorithm~\ref{algo:sdp}). 

\begin{algorithm}[H]
\caption{Find a Robust Projection}
\label{algo:spca}
\begin{algorithmic}[1]
\Function{RobustProjection}{data $A \in R^{m \times n}$, rank $k$, reconstruction error $\delta$} 
\State Set $M:=(A^\top A)/\tr(A^\top A)$. Initialize $\widehat{\Pi} \gets \emptyset, \widehat{\kappa}=\infty$. 
\For{different values of $r \le k$}
\State Find $r$-PCA of the $M$ to get a rank $r$ projection $\Pi_1$. $M' \gets M - \Pi_1 M \Pi_1$  
\State Run sparse PCA on $M'$ to find a rank $(k-r)$ projection $\Pi_2$. Set $\Pi=\Pi_1 + \Pi_2$. 
\State Run $\textsc{CertifySDP}(\Pi,\delta=1/4,\rho)$ to get an upper bound $\kappa$.
\If{$\kappa < \widehat{\kappa}$ and if $\iprod{M, I-\Pi} \le \delta$}
    \State $\widehat{\Pi} \gets \Pi$, $\widehat{\kappa}=\kappa$.
\EndIf
\EndFor
\State Output $\widehat{\Pi}$, $\widehat{\kappa}$. 
\EndFunction
\end{algorithmic}
\end{algorithm}
\fi

\section{Training Certified $\ell_\infty$ Robust Networks in Natural Representations.}
\label{sec:imperceptibiity}

Building upon our theoretical results from the previous section we now demonstrate that for natural representations, one can indeed achieve better certified robustness to $\ell_\infty$ perturbations by translating guarantees from certified $\ell_2$ robustness. We focus on image data, and study the representation of images in the DCT basis.  
Before we describe the details of our training and certification procedure for $\ell_\infty$ robustness, we provide further empirical evidence that imperceptibility in natural representations such as the DCT basis is a desirable property.

\paragraph{Study of imperceptibility in DCT basis.}
\begin{figure}[h]
    \centering
    \includegraphics[width=0.9\textwidth]{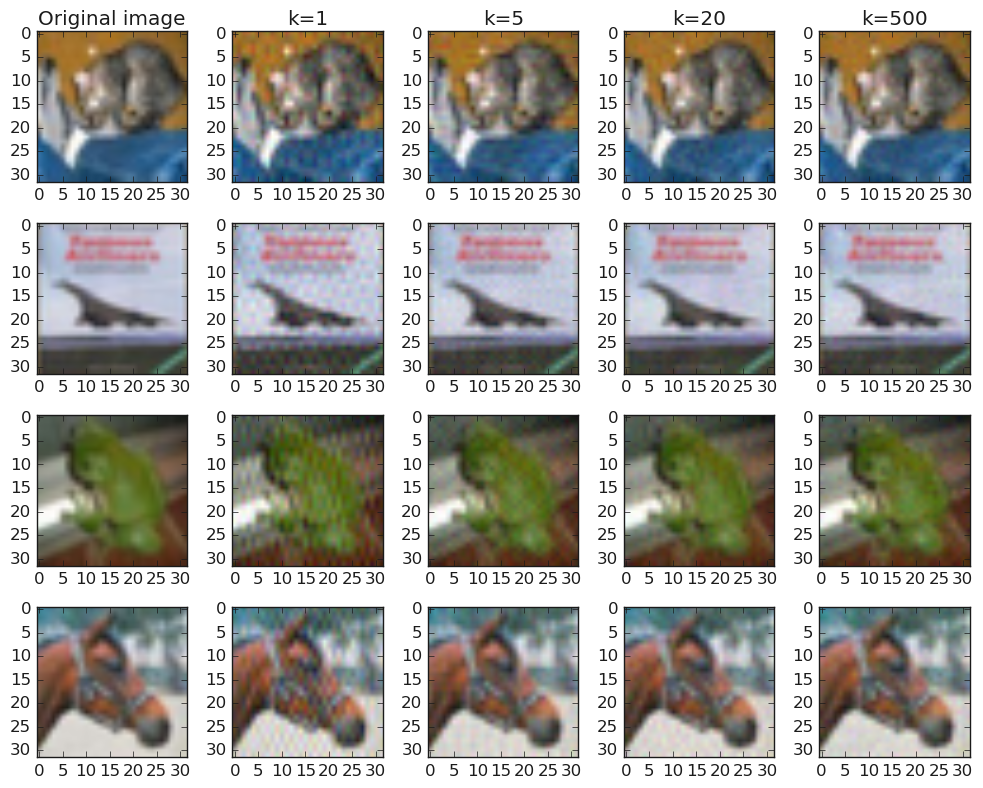}
    \caption{\label{fig:adv_examples_dct} Original images (leftmost) from the CIFAR-10 dataset and their perturbed versions when sparse random perturbations are added in the DCT basis with sparsity $k$. Large perturbations in the DCT basis (e.g., $k=1$ ) lead to perceptible changes in the pixel space though they are $\ell_\infty$ perturbations of $\epsilon \leq 0.09$. As $k$ increases the imperceptibility of the perturbed images improves.}
\end{figure}
We argue that for adversarial perturbations to be imperceptible to humans they should be of small magnitude in the DCT representation~(perhaps in addition to being small in the pixel basis). 
We take images from the CIFAR-10 dataset and transform them into the DCT basis. We then add sparse random perturbations to them. In particular, for a sparsity parameter $k$, we pick $k$ coordinates in the DCT basis at random and add a random perturbation with $\ell_\infty$ norm of $c_k/\sqrt{k}$ where $c_k \approx \eps \sqrt{n}$ is chosen such that the perturbed images are $\epsilon \leq 0.09$ away from the unperturbed images in the pixel space. Notice that for small values of $k$, a perturbation of large $\ell_\infty$ norm is added. Figure~\ref{fig:adv_examples_dct} visualizes the perturbed images for different values of $k$. As seen, large perturbations in the DCT basis lead to visually perceptible changes, even if they are ($\epsilon \leq 0.09$-)close in the pixel basis. For comparison we also include in Appendix~\ref{sec:app-imperceptibility} imperceptible adversarial examples for these images that were generated via the PGD based method of \cite{madry2017towards} on a ResNet-32 network trained on the CIFAR-10 dataset  for robustness to $\ell_\infty$ perturbations of magnitude $\epsilon = 0.09$. 
This further motivates studying robustness in natural data representations. 


\paragraph{Training certified $\ell_\infty$ robust networks in the DCT basis.}

The methods developed in Section~\ref{sec:linf:methods} and ~\ref{sec:l2} together give algorithms for training classifiers with certified $\ell_\infty$ robustness. However while we want $\ell_\infty$ robustness in a different representation $\calU$ (e.g., in the DCT basis), it may still be more convenient to use off-the-shelf methods for performing the training in the original representation $\calX \subseteq \R^n$ (e.g., pixel representation).  
Let the orthogonal matrix $O \in \R^{n \times n}$ represent the DCT transformation. Consider an input $x \in \calX$ and let $u=Ox \in \calU$ be its DCT representation, where $\calU$ is the space of images in the DCT basis. It is easy to see that functions $f:\calX \to \calY$ and $g:\calU \to \calY$ given by $g(u):=f(O^{-1}u)=f(x)$ have the same certified accuracy to $\ell_2$ perturbations.
Moreover 
if the classifier $g$ satisfies $ g(u)=g(\tPi u)$ for some projection $\tPi$, the robust accuracy of $g$ to $\ell_\infty$ perturbations in the representation $\calU$ satisfies (see Proposition~\ref{prop:rotation} in appendix) 
\begin{equation} 
\label{eq:rotationandprojection}
    \accu_{\epsilon'}^{(\ell_\infty)}(g) \ge \accu_{\epsilon}^{(\ell_2)}(f), \text{  for any }\epsilon>0, 0 \le \epsilon' \le \epsilon/\norm{\tPi}_{\infty \to 2}.
\end{equation}

\begin{figure}[h]
\vspace{-20pt}

\centering
\subfloat{\includegraphics[width=6.2cm]{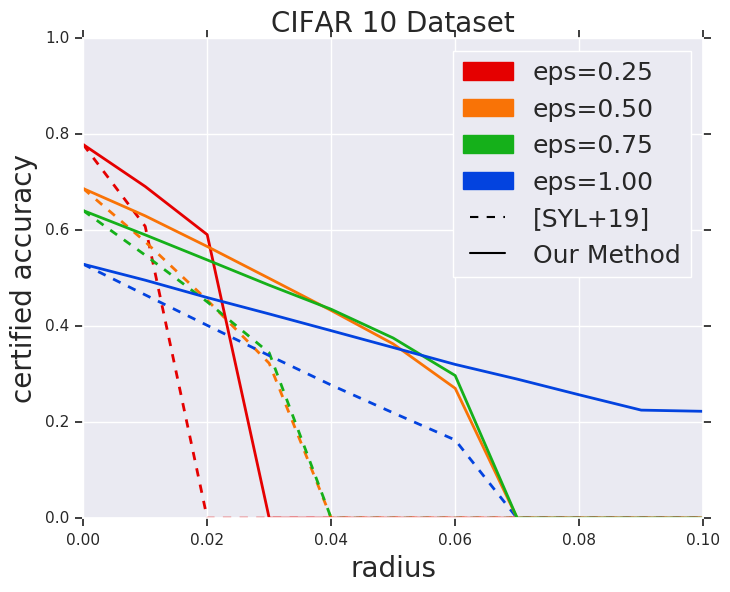}}\hfil
\subfloat{\includegraphics[width=6.2cm]{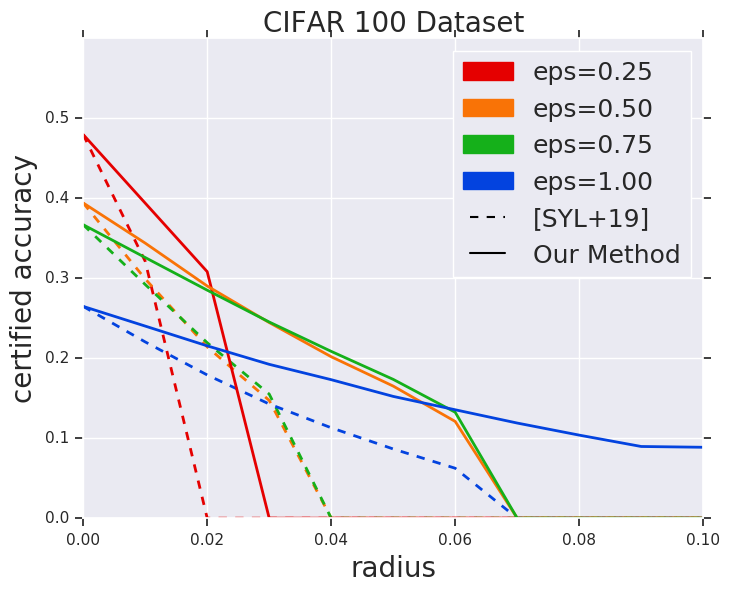}}\hfil
\caption{\label{fig:linf} A comparison of certified $\ell_\infty$ accuracy (y-axis) in the DCT basis 
of our method to that of~\cite{salman2019provably}, for different values of $\epsilon$ and with different certified radii on the x-axis, for $\lambda=0.5$. The left and right plots describe results for the CIFAR-10 and the CIFAR-100 datasets respectively. }

\vskip -.15in
\end{figure}


\paragraph{Experimental Data.} We evaluate our approach on the CIFAR-10 and CIFAR-100 datasets. From \eqref{eq:rotationandprojection}, it is sufficient to train a classifier in the original pixel space with an appropriate projection $\tPi'=O\Pi$.  Hence, we train a smoothed classifier as defined in \eqref{eq:smoothed-classifier} using Algorithm~\ref{algo:l2-training}. 
To obtain the required $\tPi$, we first use the sparse PCA based heuristic (Algorithm~\ref{algo:spca}) to find a projection matrix of rank $200$ for the three image channels separately. We then use Algorithm~\ref{algo:sdp} to compute upper bounds on the ${\infty \to 2}$ operator norm of the projections matrices. Finally, we combine the obtained projection matrices from each channel to obtain a projection $\tPi$. Table~\ref{tbl:sparse-pca} shows the values of the operator norms certified by our algorithm for each image channel and for the combined projection matrix. Notice the obtained subspaces have operator norm values significantly smaller than $\sqrt{n}=55.42$. The reconstruction error in each case, when projected onto $\tPi$ is at most $0.0345$.



\begin{wraptable}{r}{0pt}

\begin{tabular}{ |c|c|c|c|c| }
\hline
Dataset & R & G & B & $\tPi$\\
\hline
CIFAR-10 & $17.45$ & $17.51$ & $17.39$ & $30.22$\\
\hline
CIFAR-100 & $17.22$ & $17.33$ & $17.37$ & $29.97$\\
\hline
\end{tabular}
 \caption{\label{tbl:sparse-pca}The table shows bounds on $\infty \to 2$ norm for projection matrices obtained by Algorithm~\ref{algo:sdp} on CIFAR-10 and CIFAR-100 training sets.}

\vspace{-12pt}
\end{wraptable}

 

 After training, on an input $x$ we obtain a certified radius for $\ell_\infty$ perturbations in the DCT basis by obtaining a certified $\ell_2$ radius of $\epsilon$ via randomized smoothing and then dividing the obtained value by $\|\tPi\|_{\infty \to 2}$. We compare with the approach of~\cite{salman2019provably} for training a smoothed classifier without projections. Since the classifier of~\cite{salman2019provably} does not involve projections, we translate the resulting $\ell_2$ robustness guarantee into an $\ell_\infty$ guarantee by dividing with $\sqrt{n} = 55.42$ as is done in~\cite{salman2019provably}.
 Figure~\ref{fig:linf} shows that across a range of training parameters, our proposed approach via robust projections leads to significantly higher certified accuracy to $\ell_\infty$ perturbations in the DCT basis.



\section{Related Work}
\label{app:related}
Several recent works have aimed to design algorithms that are robust to adversarial perturbations. This is a rapidly growing research area, and here we survey the works most relevant in the context of the current paper. 

In practice, first order methods are a popular choice for adversarial training.
Algorithms such as the {\em fast gradient sign method} (FGSM)~\cite{goodfellow2014explaining} and {\em projected gradient descent} (PGD)~\cite{madry2017towards} fall into this category. These methods are appealing since they are generally applicable to many network architecture and only rely on black box access to gradient information. Recent works have aimed to improve upon the scalability and runtime efficiency of the above approaches~\cite{zhang2019theoretically, shafahi2019adversarial, wong2020fast}.
While these methods can be used to gather evidences regarding the robustness of a classifier to first order attacks, i.e., those based on gradient information, they do not exclude the possibility of a stronger attack that can significantly degrade their performance.

There have been recent works studying the problem of efficiently certifying the adversarial robustness of classifiers from both empirical and theoretical perspectives~\cite{croce2018provable, wang1811mixtrain, wang2018efficient, gehr2018ai2, mirman2018differentiable,singh2018fast,weng2018towards,zhang2018efficient, awasthi2019robustness}. For the case of certifying robustness to $\ell_\infty$ perturbations, recent works have explored linear programming~(LP) and semi-definite programming~(SDP) based approaches to produce lower bounds on the certified robustness of neural networks. Such approaches have seen limited success beyond shallow (depth at most 4) networks~\cite{raghunathan2018semidefinite,wong2018provable}. These works use convex programming (including SDPs) to directly provide an upper bound on the objective that corresponds to finding an adversarial perturbation for a given classifier $f$ on input $x$. While our multiplicative weights method also approximately solves a convex program (specifically, an SDP), we use the convex program to translate certified $\ell_2$ robustness guarantees to certified $\ell_\infty$ guarantees in a black box manner. 

Another line of work has focused on methods for certifying robustness via propagating interval bounds~\cite{gowal2018effectiveness}. While these methods have seen recent success in scaling up to large networks, they typically cannot be applied in a black box manner and require access to the structure of the underlying network.

Randomized smoothing proposed in the works of~\cite{liu2018towards,lecuyer2019certified, cohen2019certified, cao2017mitigating} is a simple and effective way to certify robustness of neural networks to $\ell_2$ perturbations. These methods work by creating a smoothed classifier by adding Gaussian noise to inputs. The improved Lipschitzness of the smoothed classifier can then be translated to bounds on certified $\ell_2$ robustness. Other recent works have also proposed developing Lipschitz classifiers in order to get improved robustness~\cite{yang2020adversarial}.

The recent work of~\cite{salman2019provably}, that we build upon, shows that by incorporating the smoothed classifier into the 
training process, one can get state-of-the art results for 
certified $\ell_2$ robustness. The recent works of~\cite{li2019certified, dvijotham2018training, yang2020randomized} extends the work of~\cite{salman2019provably} by developing smoothing based methods for other $\ell_p$ norms and noise distribution other than the Gaussian distribution. The authors in~\cite{salman2020black} explore smoothing based methods for making a pretrained classifier certifiably robust.

While smoothing based methods have been successful for providing robustness to $\ell_2$ perturbations, they have not demonstrated the same benefits for for the more challenging setting of robustness to $\ell_\infty$ perturbations. The recent works of~\cite{blum2020random, yang2020randomized, kumar2020curse} have provided lower bounds demonstrating that smoothing based methods are unlikely to yield benefits for robustness to $\ell_\infty$ perturbations in the worst case. There have also been recent works exploiting low rank representations to achieve better adversarial robustness~\cite{yang2019me} empirically against first order attacks such as the PGD method~\cite{madry2017towards}. The goal in our work however, is to use low rank representations to achieve {\em certified} accuracy guarantees. 

{\em Fast approximate SDP algorithms.} Our certification algorithm is based on the multiplicative weights update (MWU) framework applied to the problem of Quadratic Programming. Klein and Lu~\cite{kleinLu} gave the first MWU-based algorithm for solving particular semidefinite programs (SDP). They used the framework of \cite{plotkinST} for solving convex programs to give a faster algorithm for solving the maxcut SDP within a $1+\delta$ approximation using $\tilde{O}(n)$ eigenvalue/eigenvector computations, where the $\tilde{O}$ hides polylogarithmic factors in $n$ and polynomial factors in $1/\delta$.  But the analysis of Klein and Lu ~\cite{kleinLu} is specialized for the Max-Cut problem, which is a special case of \eqref{eq:operator:problem} where $M$ is a graph Laplacian; it does not directly extend to our more general setting (they use diagonal dominance of a Laplacian to get a small bound on the width). Arora, Hazan and Kale~\cite{AroraHK, Kalethesis} introduced a more general framework for solving semidefinite programs. In particular, they showed how when this framework is applied to the problem of Quadratic Programming, it gives a running time bound of $O(\frac{n^{1.5}}{\delta^{2.5}} \cdot \min\set{m,n^{1.5}/(\delta \alpha^*)})$, where the optimal solution value is $\sdpval=\alpha^* \norm{M}_1$, and $m$ is the number of non-zeros of matrix $M$. For SDPs of the form \eqref{eq:primal}, the width parameter could be reasonably large and could result in $\tilde{\Omega}(n^{3/2})$ iterations (see Section 6.3 of \cite{Kalethesis}). To the best of our knowledge this represents the prior best running time for approximately solving the Quadratic Programming SDP~\eqref{eq:QP:problem}, and is most related to our work. 

Arora et al.~\cite{AroraK07} also gave primal-dual based algorithms near-linear time algorithms for several combinatorial problems like max-cut and sparsest cut that use semidefinite programming relaxations. Iyengar et al.~\cite{iyengarPS11, jainY11,ALO16} consider positive covering and packing SDPs like max-cut, and give fast approximate algorithms based on the multiplicative weights method. In particular, \cite{jainY11, ALO16, swatiSTOC2020} and several other works give parallel algorithms, where the iteration count only has a mild polylogarithmic dependence on width for solving these positive SDPs. This uses the matrix multiplicative weights method that involves computing matrix exponentials. Quadratic Programming does not fall into this class of problems in general, unless $M \succeq 0$. A very recent paper of Lee and Padmanabhan~\cite{SY19} gives algorithms that work for problems like quadratic programming with diagonal constraints; however this gives an additive approximation to the objective which does not suffice for our purposes. Moreover, our algorithm is very simple and practical, and each iteration only uses a single computation of the largest eigenvalue. Hence, analyzing this algorithm is interesting in its own right.  Finally, interior point methods, cutting plane methods and the Ellipsoid algorithm find solutions to the SDP that have much better dependence on the accuracy $\delta$ e.g., a polynomial in $\log(1/\delta)$ dependence, at the expense of significantly higher dependence on $n$~\cite{Alizadeh, LSW15}. 

\vspace{-10pt}
\section{Conclusion}

In this paper, we have shown significant benefits in leveraging natural structure that exists in real-world data e.g., low-rank or sparse representations, for obtaining certified robustness guarantees under both $\ell_2$ perturbations and $\ell_\infty$ perturbations in natural data representations. Our experiments involving imperceptibility in the DCT basis for images suggest that further study of $\ell_\infty$ robustness for other natural basis (apart from the co-ordinate basis) would be useful for different data domains like images, audio etc. We also gave faster algorithms for approximately solving semi-definite programs for quadratic programming (with provable guarantees that improve the state-of-the-art), to obtain certified $\ell_\infty$ robustness guarantees. Such problem-specific fast approximate algorithms for powerful algorithmic techniques like SDPs and other convex relaxations may lead to more scalable certification procedures with better guarantees. Finally it would be interesting to see if our ideas and techniques involving the $\infty \to 2$ operator norm can be adapted into the training phase, in order to achieve better certified $\ell_\infty$ robustness in any desired basis without compromising much on natural accuracy.

\ifNeurIPS
\section*{Broader Impact}
Our work provides efficient algorithms for training neural networks with certified robustness guarantees. This can have significant positive societal impact considering the importance of protecting AI systems against malicious adversaries. A classifier with certified robustness guarantees can give a sense of security to the end user. On the other hand, our methods achieve robustness at the expense of a small loss in natural test accuracy as compared to non-adversarial training. It is unclear how this loss in accuracy is distributed across the population. This could have a negative societal impact if the loss in accuracy is disproportionately on data points/individuals belonging to a specific demographic group based on say race or gender. That said, robustness to perturbations also corresponds to a natural notion of individual fairness since data points with similar features need to be treated similarly by a robust classifier. Hence, a careful study must be done to understand these effects before a large scale practical deployment of systems based on our work.
\fi

\bibliographystyle{alpha}
\bibliography{advrobustness}

\newcommand{\etalchar}[1]{$^{#1}$}
\begin{thebibliography}{GMDC{\etalchar{+}}18}

\bibitem[ACCV19]{ACCV}
Pranjal Awasthi, Vaggos Chatziafratis, Xue Chen, and Aravindan Vijayaraghavan.
\newblock Adversarially robust low dimensional representations.
\newblock {\em arXiv preprint arXiv:1911.13268}, 2019.

\bibitem[ADV19]{awasthi2019robustness}
Pranjal Awasthi, Abhratanu Dutta, and Aravindan Vijayaraghavan.
\newblock On robustness to adversarial examples and polynomial optimization.
\newblock In {\em Advances in Neural Information Processing Systems}, pages
  13737--13747, 2019.

\bibitem[AHK05]{AroraHK}
Sanjeev Arora, Elad Hazan, and Satyen Kale.
\newblock Fast algorithms for approximate semidefinite programming using the
  multiplicative weights update method.
\newblock In {\em 46th Annual IEEE Symposium on Foundations of Computer Science
  (FOCS'05)}, pages 339--348. IEEE, 2005.

\bibitem[AK07]{AroraK07}
Sanjeev Arora and Satyen Kale.
\newblock A combinatorial, primal-dual approach to semidefinite programs.
\newblock In {\em Proceedings of the Thirty-Ninth Annual ACM Symposium on
  Theory of Computing}, STOC ’07, page 227–236, New York, NY, USA, 2007.
  Association for Computing Machinery.

\bibitem[Ali95]{Alizadeh}
Farid Alizadeh.
\newblock Interior point methods in semidefinite programming with applications
  to combinatorial optimization.
\newblock {\em SIAM Journal on Optimization}, 5(1):13--51, 1995.

\bibitem[AMMN06]{alon2006quadratic}
Noga Alon, Konstantin Makarychev, Yury Makarychev, and Assaf Naor.
\newblock Quadratic forms on graphs.
\newblock {\em Inventiones mathematicae}, 163(3):499--522, 2006.

\bibitem[AN04]{alon2004approximating}
Noga Alon and Assaf Naor.
\newblock Approximating the cut-norm via grothendieck's inequality.
\newblock In {\em Proceedings of the thirty-sixth annual ACM symposium on
  Theory of computing}, pages 72--80. ACM, 2004.

\bibitem[ApS19]{mosek}
MOSEK ApS.
\newblock {\em The MOSEK Fusion API for Python manual Version 9.1.13}, 2019.

\bibitem[AZLO16]{ALO16}
Zeyuan Allen-Zhu, Yin~Tat Lee, and Lorenzo Orecchia.
\newblock Using optimization to obtain a width-independent, parallel, simpler,
  and faster positive sdp solver.
\newblock In {\em Proceedings of the Twenty-Seventh Annual ACM-SIAM Symposium
  on Discrete Algorithms}, SODA ’16, page 1824–1831, USA, 2016. Society for
  Industrial and Applied Mathematics.

\bibitem[BCM{\etalchar{+}}13]{biggio2013evasion}
Battista Biggio, Igino Corona, Davide Maiorca, Blaine Nelson, Nedim
  {\v{S}}rndi{\'c}, Pavel Laskov, Giorgio Giacinto, and Fabio Roli.
\newblock Evasion attacks against machine learning at test time.
\newblock In {\em Joint European conference on machine learning and knowledge
  discovery in databases}, pages 387--402. Springer, 2013.

\bibitem[BDMZ20]{blum2020random}
Avrim Blum, Travis Dick, Naren Manoj, and Hongyang Zhang.
\newblock Random smoothing might be unable to certify {$\ell_\infty$}
  robustness for high-dimensional images.
\newblock {\em arXiv preprint arXiv:2002.03517}, 2020.

\bibitem[BGG{\etalchar{+}}18]{bhattiprolu2018inapproximability}
Vijay Bhattiprolu, Mrinalkanti Ghosh, Venkatesan Guruswami, Euiwoong Lee, and
  Madhur Tulsiani.
\newblock Inapproximability of matrix p to q norms.
\newblock {\em arXiv preprint arXiv:1802.07425}, 2018.

\bibitem[CAH18]{croce2018provable}
Francesco Croce, Maksym Andriushchenko, and Matthias Hein.
\newblock Provable robustness of relu networks via maximization of linear
  regions.
\newblock {\em arXiv preprint arXiv:1810.07481}, 2018.

\bibitem[CG17]{cao2017mitigating}
Xiaoyu Cao and Neil~Zhenqiang Gong.
\newblock Mitigating evasion attacks to deep neural networks via region-based
  classification.
\newblock In {\em Proceedings of the 33rd Annual Computer Security Applications
  Conference}, pages 278--287, 2017.

\bibitem[CRK19]{cohen2019certified}
Jeremy~M Cohen, Elan Rosenfeld, and J~Zico Kolter.
\newblock Certified adversarial robustness via randomized smoothing.
\newblock {\em arXiv preprint arXiv:1902.02918}, 2019.

\bibitem[CW04]{charikarwirth}
Moses Charikar and Anthony Wirth.
\newblock Maximizing quadratic programs: Extending grothendieck’s inequality.
\newblock In {\em Proceedings of the 45th Annual IEEE Symposium on Foundations
  of Computer Science}, FOCS ’04, page 54–60, USA, 2004. IEEE Computer
  Society.

\bibitem[CW18]{carlini2018audio}
Nicholas Carlini and David Wagner.
\newblock Audio adversarial examples: Targeted attacks on speech-to-text.
\newblock In {\em 2018 IEEE Security and Privacy Workshops (SPW)}, pages 1--7.
  IEEE, 2018.

\bibitem[DB16]{cvxpy}
Steven Diamond and Stephen Boyd.
\newblock {CVXPY}: {A} {P}ython-embedded modeling language for convex
  optimization.
\newblock {\em Journal of Machine Learning Research}, 17(83):1--5, 2016.

\bibitem[DGS{\etalchar{+}}18]{dvijotham2018training}
Krishnamurthy Dvijotham, Sven Gowal, Robert Stanforth, Relja Arandjelovic,
  Brendan O'Donoghue, Jonathan Uesato, and Pushmeet Kohli.
\newblock Training verified learners with learned verifiers.
\newblock {\em arXiv preprint arXiv:1805.10265}, 2018.

\bibitem[ERLD17]{ebrahimi2017hotflip}
Javid Ebrahimi, Anyi Rao, Daniel Lowd, and Dejing Dou.
\newblock Hotflip: White-box adversarial examples for text classification.
\newblock {\em arXiv preprint arXiv:1712.06751}, 2017.

\bibitem[GDS{\etalchar{+}}18]{gowal2018effectiveness}
Sven Gowal, Krishnamurthy Dvijotham, Robert Stanforth, Rudy Bunel, Chongli Qin,
  Jonathan Uesato, Relja Arandjelovic, Timothy Mann, and Pushmeet Kohli.
\newblock On the effectiveness of interval bound propagation for training
  verifiably robust models.
\newblock {\em arXiv preprint arXiv:1810.12715}, 2018.

\bibitem[GMDC{\etalchar{+}}18]{gehr2018ai2}
Timon Gehr, Matthew Mirman, Dana Drachsler-Cohen, Petar Tsankov, Swarat
  Chaudhuri, and Martin Vechev.
\newblock Ai2: Safety and robustness certification of neural networks with
  abstract interpretation.
\newblock In {\em 2018 IEEE Symposium on Security and Privacy (SP)}, pages
  3--18. IEEE, 2018.

\bibitem[GSS14]{goodfellow2014explaining}
Ian~J Goodfellow, Jonathon Shlens, and Christian Szegedy.
\newblock Explaining and harnessing adversarial examples.
\newblock {\em arXiv preprint arXiv:1412.6572}, 2014.

\bibitem[HZRS16]{he2016deep}
Kaiming He, Xiangyu Zhang, Shaoqing Ren, and Jian Sun.
\newblock Deep residual learning for image recognition.
\newblock In {\em Proceedings of the IEEE conference on computer vision and
  pattern recognition}, pages 770--778, 2016.

\bibitem[IPS11]{iyengarPS11}
G.~Iyengar, D.~J. Phillips, and C.~Stein.
\newblock Approximating semidefinite packing programs.
\newblock {\em SIAM Journal on Optimization}, 21(1):231--268, 2011.

\bibitem[JLL{\etalchar{+}}20]{swatiSTOC2020}
Arun Jambulapati, Yin~Tat Lee, Jerry Li, Swati Padmanabhan, and Kevin Tian.
\newblock Positive semidefinite programming: mixed, parallel, and
  width-independent.
\newblock In {\em Proccedings of the 52nd Annual {ACM} {SIGACT} Symposium on
  Theory of Computing, {STOC} 2020, Chicago, IL, USA, June 22-26, 2020}, pages
  789--802. {ACM}, 2020.

\bibitem[JY11]{jainY11}
Rahul Jain and Penghui Yao.
\newblock A parallel approximation algorithm for positive semidefinite
  programming.
\newblock In {\em Proceedings of the 2011 IEEE 52nd Annual Symposium on
  Foundations of Computer Science}, FOCS ’11, page 463–471, USA, 2011. IEEE
  Computer Society.

\bibitem[Kal07]{Kalethesis}
Satyen Kale.
\newblock {\em Efficient algorithms using the multiplicative weights update
  method}.
\newblock Princeton University, 2007.

\bibitem[KH{\etalchar{+}}09]{krizhevsky2009learning}
Alex Krizhevsky, Geoffrey Hinton, et~al.
\newblock Learning multiple layers of features from tiny images.
\newblock 2009.

\bibitem[KL96]{kleinLu}
Philip Klein and Hsueh-I Lu.
\newblock Efficient approximation algorithms for semidefinite programs arising
  from max cut and coloring.
\newblock In {\em Proceedings of the Twenty-Eighth Annual ACM Symposium on
  Theory of Computing}, STOC ’96, page 338–347, New York, NY, USA, 1996.
  Association for Computing Machinery.

\bibitem[KLGF20]{kumar2020curse}
Aounon Kumar, Alexander Levine, Tom Goldstein, and Soheil Feizi.
\newblock Curse of dimensionality on randomized smoothing for certifiable
  robustness.
\newblock {\em arXiv preprint arXiv:2002.03239}, 2020.

\bibitem[LAG{\etalchar{+}}19]{lecuyer2019certified}
Mathias Lecuyer, Vaggelis Atlidakis, Roxana Geambasu, Daniel Hsu, and Suman
  Jana.
\newblock Certified robustness to adversarial examples with differential
  privacy.
\newblock In {\em 2019 IEEE Symposium on Security and Privacy (SP)}, pages
  656--672. IEEE, 2019.

\bibitem[LCWC19]{li2019certified}
Bai Li, Changyou Chen, Wenlin Wang, and Lawrence Carin.
\newblock Certified adversarial robustness with additive noise.
\newblock In {\em Advances in Neural Information Processing Systems}, pages
  9459--9469, 2019.

\bibitem[LCZH18]{liu2018towards}
Xuanqing Liu, Minhao Cheng, Huan Zhang, and Cho-Jui Hsieh.
\newblock Towards robust neural networks via random self-ensemble.
\newblock In {\em Proceedings of the European Conference on Computer Vision
  (ECCV)}, pages 369--385, 2018.

\bibitem[LP19]{SY19}
Yin~Tat Lee and Swati Padmanabhan.
\newblock An {\~{o}}(m/{\(\epsilon\)}\({}^{\mbox{3.5}}\))-cost algorithm for
  semidefinite programs with diagonal constraints.
\newblock {\em CoRR}, abs/1903.01859, 2019.

\bibitem[LSW15]{LSW15}
Y.~T. {Lee}, A.~{Sidford}, and S.~C. {Wong}.
\newblock A faster cutting plane method and its implications for combinatorial
  and convex optimization.
\newblock In {\em 2015 IEEE 56th Annual Symposium on Foundations of Computer
  Science}, pages 1049--1065, 2015.

\bibitem[MBPS09]{mairal2009sparsePCA}
Julien Mairal, Francis Bach, Jean Ponce, and Guillermo Sapiro.
\newblock Online dictionary learning for sparse coding.
\newblock In {\em Proceedings of the 26th Annual International Conference on
  Machine Learning}, ICML ’09, page 689–696, New York, NY, USA, 2009.
  Association for Computing Machinery.

\bibitem[MGV18]{mirman2018differentiable}
Matthew Mirman, Timon Gehr, and Martin Vechev.
\newblock Differentiable abstract interpretation for provably robust neural
  networks.
\newblock In {\em International Conference on Machine Learning}, pages
  3578--3586, 2018.

\bibitem[MMS{\etalchar{+}}17]{madry2017towards}
Aleksander Madry, Aleksandar Makelov, Ludwig Schmidt, Dimitris Tsipras, and
  Adrian Vladu.
\newblock Towards deep learning models resistant to adversarial attacks.
\newblock {\em arXiv preprint arXiv:1706.06083}, 2017.

\bibitem[Nes98]{nesterov1998semidefinite}
Yu~Nesterov.
\newblock Semidefinite relaxation and nonconvex quadratic optimization.
\newblock {\em Optimization methods and software}, 9(1-3):141--160, 1998.

\bibitem[PST91]{plotkinST}
S.~A. {Plotkin}, D.~B. {Shmoys}, and E.~{Tardos}.
\newblock Fast approximation algorithms for fractional packing and covering
  problems.
\newblock In {\em [1991] Proceedings 32nd Annual Symposium of Foundations of
  Computer Science}, pages 495--504, 1991.

\bibitem[PVG{\etalchar{+}}11]{scikit-learn}
F.~Pedregosa, G.~Varoquaux, A.~Gramfort, V.~Michel, B.~Thirion, O.~Grisel,
  M.~Blondel, P.~Prettenhofer, R.~Weiss, V.~Dubourg, J.~Vanderplas, A.~Passos,
  D.~Cournapeau, M.~Brucher, M.~Perrot, and E.~Duchesnay.
\newblock Scikit-learn: Machine learning in {P}ython.
\newblock {\em Journal of Machine Learning Research}, 12:2825--2830, 2011.

\bibitem[RSL18]{raghunathan2018semidefinite}
Aditi Raghunathan, Jacob Steinhardt, and Percy~S Liang.
\newblock Semidefinite relaxations for certifying robustness to adversarial
  examples.
\newblock In {\em Advances in Neural Information Processing Systems}, pages
  10877--10887, 2018.

\bibitem[SGM{\etalchar{+}}18]{singh2018fast}
Gagandeep Singh, Timon Gehr, Matthew Mirman, Markus P{\"u}schel, and Martin
  Vechev.
\newblock Fast and effective robustness certification.
\newblock In {\em Advances in Neural Information Processing Systems}, pages
  10802--10813, 2018.

\bibitem[SNG{\etalchar{+}}19]{shafahi2019adversarial}
Ali Shafahi, Mahyar Najibi, Mohammad~Amin Ghiasi, Zheng Xu, John Dickerson,
  Christoph Studer, Larry~S Davis, Gavin Taylor, and Tom Goldstein.
\newblock Adversarial training for free!
\newblock In {\em Advances in Neural Information Processing Systems}, pages
  3353--3364, 2019.

\bibitem[SSY{\etalchar{+}}20]{salman2020black}
Hadi Salman, Mingjie Sun, Greg Yang, Ashish Kapoor, and J~Zico Kolter.
\newblock Black-box smoothing: A provable defense for pretrained classifiers.
\newblock {\em arXiv preprint arXiv:2003.01908}, 2020.

\bibitem[SYL{\etalchar{+}}19]{salman2019provably}
Hadi Salman, Greg Yang, Jerry Li, Pengchuan Zhang, Huan Zhang, Ilya
  Razenshteyn, and Sebastien Bubeck.
\newblock Provably robust deep learning via adversarially trained smoothed
  classifiers.
\newblock {\em arXiv preprint arXiv:1906.04584}, 2019.

\bibitem[SZS{\etalchar{+}}13]{szegedy2013intriguing}
Christian Szegedy, Wojciech Zaremba, Ilya Sutskever, Joan Bruna, Dumitru Erhan,
  Ian Goodfellow, and Rob Fergus.
\newblock Intriguing properties of neural networks.
\newblock {\em arXiv preprint arXiv:1312.6199}, 2013.

\bibitem[WCAJ11]{wang1811mixtrain}
Shiqi Wang, Yizheng Chen, Ahmed Abdou, and Suman Jana.
\newblock Mixtrain: scalable training of formally robust neural networks. corr
  abs/1811.02625 (2018).
\newblock {\em arxiv. org/abs}, 1811.

\bibitem[WK18]{wong2018provable}
Eric Wong and Zico Kolter.
\newblock Provable defenses against adversarial examples via the convex outer
  adversarial polytope.
\newblock In {\em International Conference on Machine Learning}, pages
  5283--5292, 2018.

\bibitem[WPW{\etalchar{+}}18]{wang2018efficient}
Shiqi Wang, Kexin Pei, Justin Whitehouse, Junfeng Yang, and Suman Jana.
\newblock Efficient formal safety analysis of neural networks.
\newblock In {\em Advances in Neural Information Processing Systems}, pages
  6367--6377, 2018.

\bibitem[WRK20]{wong2020fast}
Eric Wong, Leslie Rice, and J~Zico Kolter.
\newblock Fast is better than free: Revisiting adversarial training.
\newblock {\em arXiv preprint arXiv:2001.03994}, 2020.

\bibitem[WZC{\etalchar{+}}18]{weng2018towards}
Tsui-Wei Weng, Huan Zhang, Hongge Chen, Zhao Song, Cho-Jui Hsieh, Duane Boning,
  Inderjit~S Dhillon, and Luca Daniel.
\newblock Towards fast computation of certified robustness for relu networks.
\newblock {\em arXiv preprint arXiv:1804.09699}, 2018.

\bibitem[YDH{\etalchar{+}}20]{yang2020randomized}
Greg Yang, Tony Duan, Edward Hu, Hadi Salman, Ilya Razenshteyn, and Jerry Li.
\newblock Randomized smoothing of all shapes and sizes.
\newblock {\em arXiv preprint arXiv:2002.08118}, 2020.

\bibitem[YRZ{\etalchar{+}}20]{yang2020adversarial}
Yao-Yuan Yang, Cyrus Rashtchian, Hongyang Zhang, Ruslan Salakhutdinov, and
  Kamalika Chaudhuri.
\newblock Adversarial robustness through local lipschitzness.
\newblock {\em arXiv preprint arXiv:2003.02460}, 2020.

\bibitem[YZKX19]{yang2019me}
Yuzhe Yang, Guo Zhang, Dina Katabi, and Zhi Xu.
\newblock Me-net: Towards effective adversarial robustness with matrix
  estimation.
\newblock {\em arXiv preprint arXiv:1905.11971}, 2019.

\bibitem[ZWC{\etalchar{+}}18]{zhang2018efficient}
Huan Zhang, Tsui-Wei Weng, Pin-Yu Chen, Cho-Jui Hsieh, and Luca Daniel.
\newblock Efficient neural network robustness certification with general
  activation functions.
\newblock In {\em Advances in neural information processing systems}, pages
  4939--4948, 2018.

\bibitem[ZYJ{\etalchar{+}}19]{zhang2019theoretically}
Hongyang Zhang, Yaodong Yu, Jiantao Jiao, Eric~P Xing, Laurent~El Ghaoui, and
  Michael~I Jordan.
\newblock Theoretically principled trade-off between robustness and accuracy.
\newblock {\em arXiv preprint arXiv:1901.08573}, 2019.

\end{thebibliography}

\newpage
\appendix


\section{Additional experiments for certified $\ell_2$ robustness}
\label{sec:app-l2}

We first discuss the setting of the hyperparameters in our experiments. The experiments in Figure~\ref{fig:acc_tradeoff_cifar10}, Figure~\ref{fig:l2}, and Figure~\ref{fig:linf} were obtained by training a ResNet-32 architecture~\cite{he2016deep}. In each case we optimize the objective in \eqref{eq:pgd-objective} with a starting learning rate of $0.1$ and decaying the learning rate by a factor of $0.1$ every $50$ epochs. Each model was trained for $150$ epochs. We follow the methodology of~\cite{salman2019provably} and approximate the soft classifier in \eqref{eq:smoothed-projected-classifier-v2} by drawing $4$ noise vectors from the Gaussian distribution $N(0,\sigma^2 I)$ and optimizing the inner maximization in \eqref{eq:pgd-objective} via $10$ steps of projected gradient descent. When training the method of~\cite{salman2019provably} we choose the value of $\sigma$ to be $0.12$ for $\epsilon = 0.25, 0.5$. For $\epsilon = 0.75$ and $\epsilon=1.0$, we choose $\sigma$ to be $0.25$ and $0.5$ respectively. When training our procedure in Algorithm~\ref{algo:l2-training} we use a higher noise of $\sigma \cdot \lambda \sqrt{d/r}$, where $\sigma$ is the corresponding noise used for the method of \cite{salman2019provably}, $d=1024$ and $r=200$. We vary the parameter $\lambda$ as discussed in Section~\ref{sec:l2}. Empirically, we find that $\lambda=0.5$ gives the best results.

We first demonstrate that real data such as images have a natural low rank structure. As an illustration, Figure \ref{fig:pca_error_vs_d} shows the relative reconstruction error for the CIFAR-10 and CIFAR-100 \cite{krizhevsky2009learning} datasets, when each of the three channels is projected onto subspaces of varying dimensions, computed via principal component analysis (PCA). As can be seen, even when projected onto $200$ dimensions, the reconstruction error remains small. 
\pnote{Use a single plot for Fig 1.}
\begin{figure}[H]
    \centering
    \begin{minipage}{0.8\textwidth}
        \centering
        \includegraphics[width=1\textwidth]{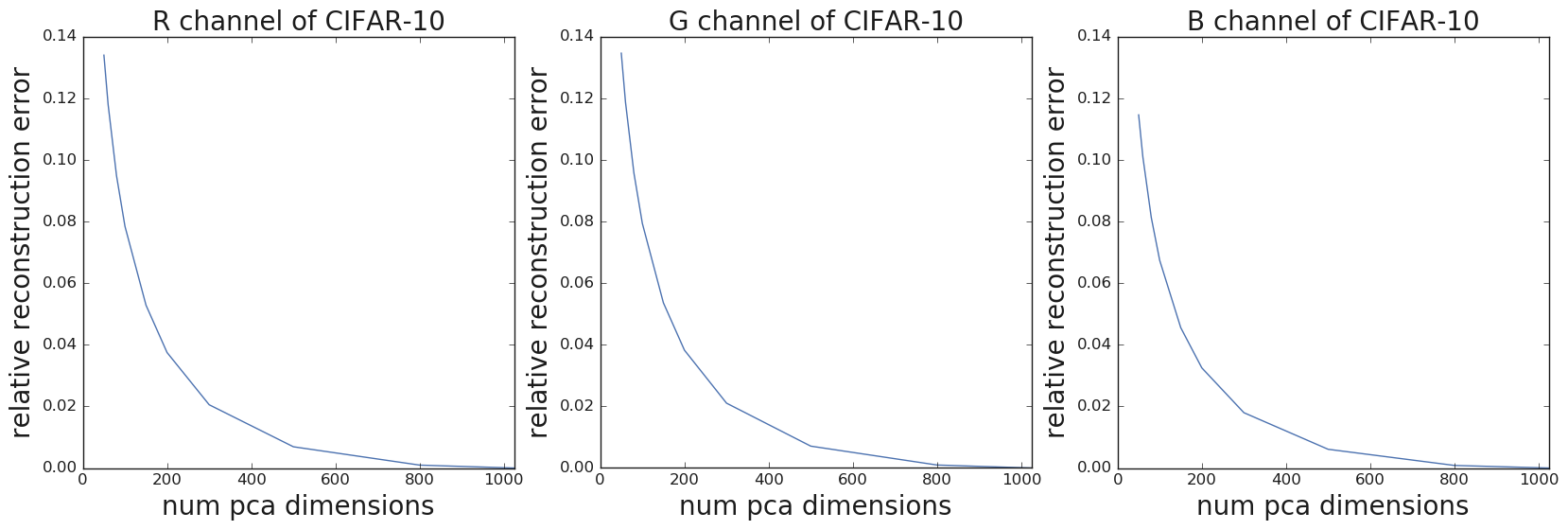} 
    \end{minipage}
    \begin{minipage}{0.8\textwidth}
        \centering
        \includegraphics[width=1\textwidth]{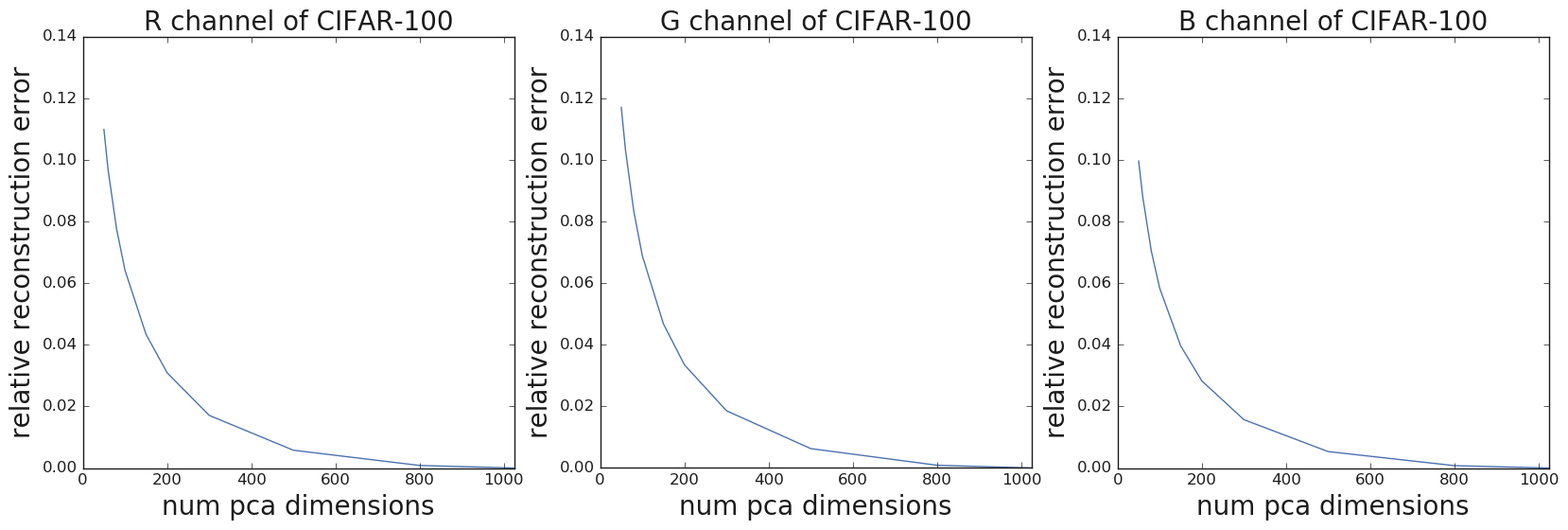} 
    \end{minipage} 
   \caption{\label{fig:pca_error_vs_d}} 
\end{figure}

Next we provide in Figure~\ref{fig:l2-app} a more fine grained comparison between our procedure in Algorithm~\ref{algo:l2-training} and the method of~\cite{salman2019provably} by separately comparing the performance of the two methods for different values of $\epsilon$ on the CIFAR-10 and CIFAR-100 datasets.
\begin{figure}[H]
\centering
\subfloat{\includegraphics[width=6cm]{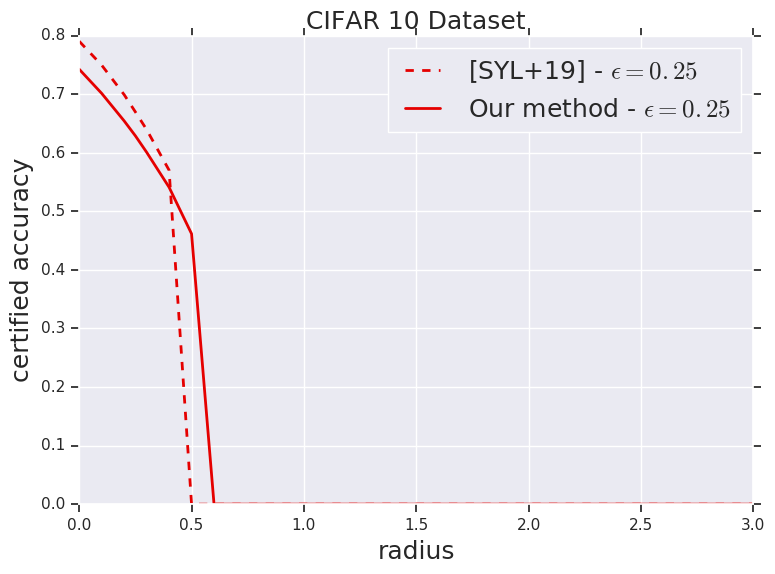}}\hfil
\subfloat{\includegraphics[width=6cm]{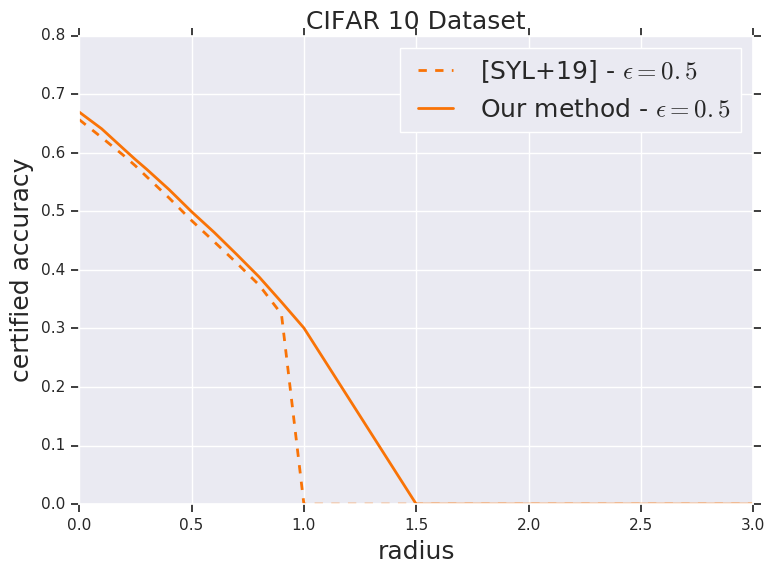}}\hfil
\subfloat{\includegraphics[width=6cm]{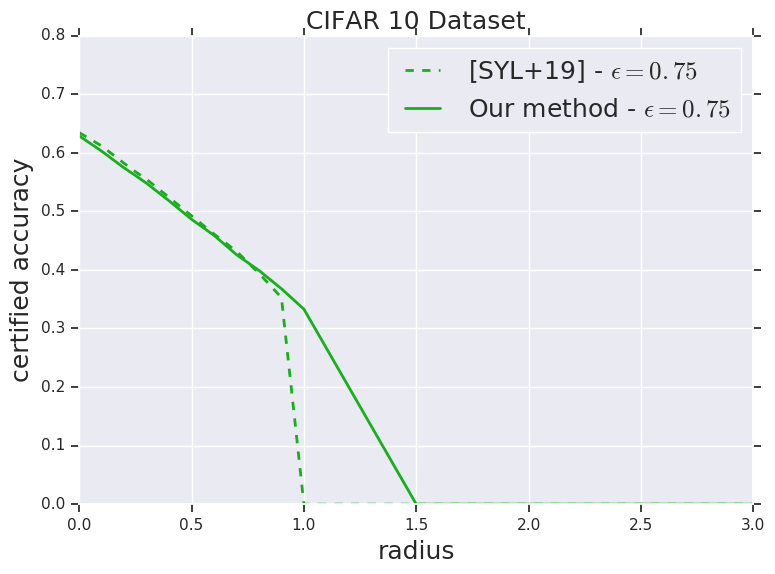}}\hfil
\subfloat{\includegraphics[width=6cm]{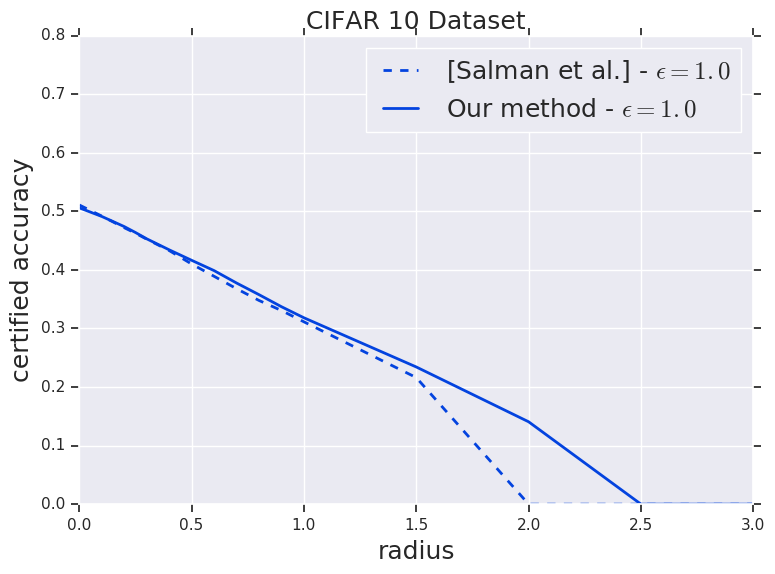}}

\subfloat{\includegraphics[width=6cm]{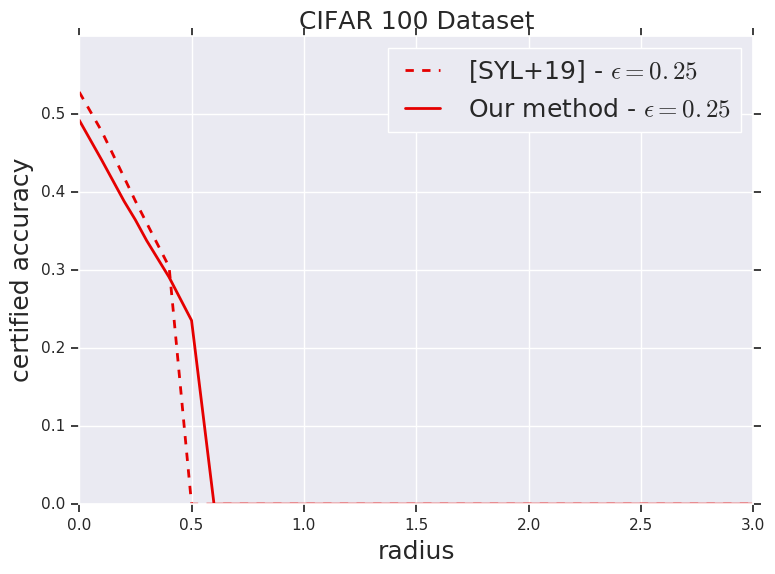}}\hfil
\subfloat{\includegraphics[width=6cm]{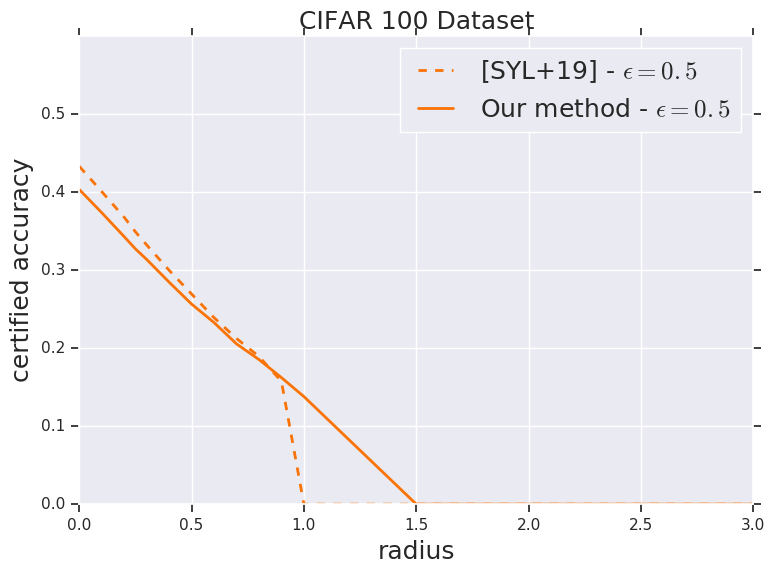}}\hfil
\subfloat{\includegraphics[width=6cm]{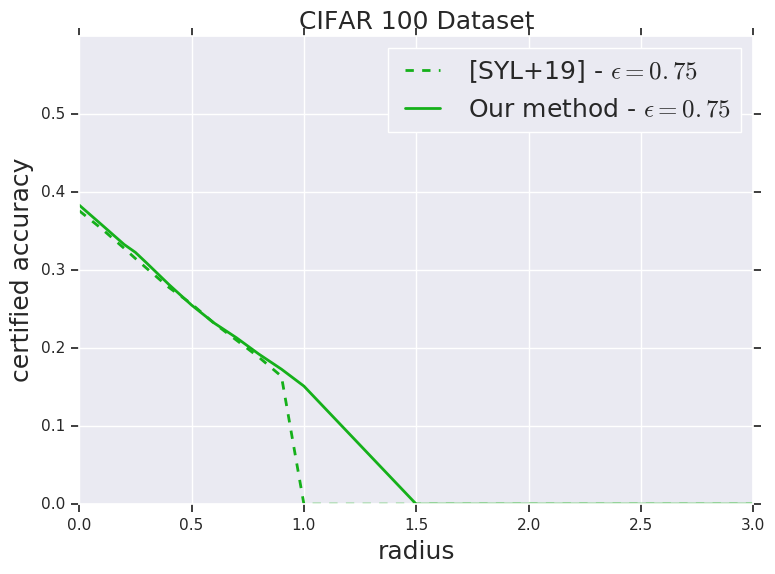}}\hfil
\subfloat{\includegraphics[width=6cm]{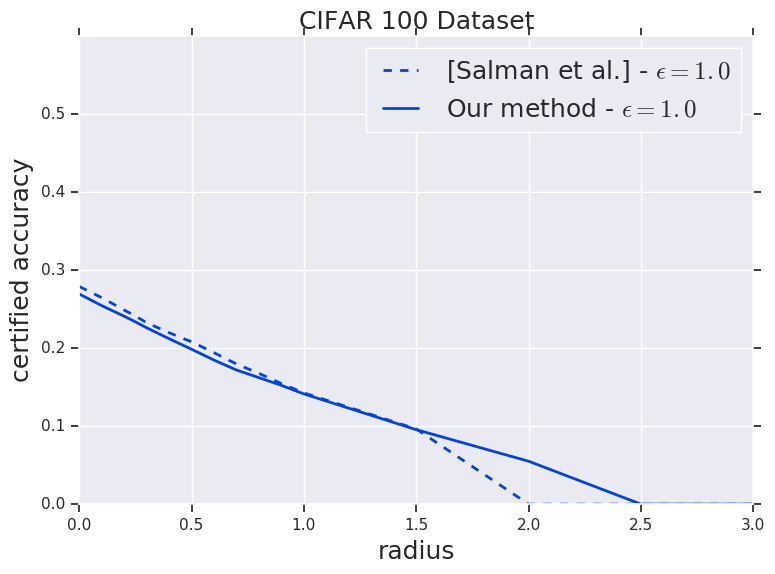}}

\caption{\label{fig:l2-app} A comparison of certified radius guarantees obtained via Algorithm~\ref{algo:l2-training} as compared to the approach of~\cite{salman2019provably}. The x-axis is the radius, and the y-axis represents the certified accuracy. }
\end{figure}

\section{Simple Theoretical Propositions and Proofs} \label{sec:simpleclaims}
The following proposition shows the equivalence between the smoothed classifiers in \eqref{eq:smoothed-projected-classifier} and \eqref{eq:smoothed-projected-classifier-v2}.

\begin{proposition}
\label{prop:l2-equivalence-app}
Given a base classifier $f: \R^n \to \mathcal{Y}$ and a projection matrix $\Pi$, on any input $x$, the smoothed classifier $g_\Pi(x)$ as defined in \eqref{eq:smoothed-projected-classifier} is equivalent to the classifier given by
\begin{align*}
    \tilde{g}_\Pi(x) = \arg \max_{y \in \mathcal{Y}} \mathbb{P}(f \big(\Pi (x+\delta) \big) = y).
    \end{align*}
    Here $\delta$ is a standard Gaussian of variance $\sigma^2$ in every direction. 
\end{proposition}
\begin{proof}
Let $\Pi$ be a projection matrix with the corresponding subspace denoted by $\mathcal{S}$.
Let $\delta$ be a random variable distributed as $N(0,\sigma^2 I)$ and $\delta_\Pi$ be a standard normal random variable with variance $\sigma^2$ within $\mathcal{S}$ and variance $0$ outside. From the property of Gaussian distributions we have that projections of a spherical Gaussian random variable with variance $\sigma^2$ in each direction are themselves Gaussian random variables with variance $\sigma^2$ in each direction within the subspace. Hence, for a fixed input $x$, the random variables $\Pi x+\delta_\Pi$ and $\Pi x+ \Pi \delta$ are identically distributed. From this we conclude that the classifier 
\begin{align*}
    g_\Pi(x) = \arg \max_{y \in \mathcal{Y}} \mathbb{P}(f(\Pi x+\delta_\Pi) = y).
    \end{align*}
    is identical to the classifier $\tilde{g}_\Pi(x)$.
\end{proof}
The following proposition shows how to obtain an $\ell_\infty$ robustness guarantee from an $\ell_2$ robustness guarantee. 
\begin{proposition}\label{prop:linftol2}
Consider any classifier $f: \R^n \to \calY$, and let the classifier $g(x):=f(\Pi x)$, where $\Pi \in \R^{n \times n}$ is any projection matrix. Then if we denote by $r_2(x), r_\infty(x)$ the radius of robustness measured in $\ell_2$ and $\ell_\infty$ norm respectively of $g$ around a point $x \in \R^n$. Then we have 
$$ r_\infty (x) \ge \frac{r_2 (x)}{\norm{\Pi}_{\infty \to 2}} , \text{ where } \norm{\Pi}_{\infty \to 2}= \max_{x: \norm{x}_\infty \le 1} \norm{\Pi x}_2. $$
Hence for any $\eps \geq 0$, we have $\accu^{(\ell_\infty)}_{\eps'}(g) \ge \accu^{(\ell_2)}_\eps(g)$,  for all $0 \le \eps' \le \eps/ \norm{P}_{\infty \to 2}$. 
\end{proposition}
\begin{proof}
Consider any perturbation $x'=x+z$ where $\norm{z}_\infty \le \eps'$ such that the predictions given by $g(x')=f(\Pi x')$
and $g(x)=f(\Pi x)$ differs. 
Consider another perturbation $z'$ such that $z' = \Pi z$. Notice that $g(x+z') = f(\Pi(x + \Pi z)) = g(x')$, since $\Pi^2 = \Pi$ for a projection matrix. Hence, the predictions of the network differ on $x$ and $x+z'$. Moreover $\norm{z'}_2 \le \norm{\Pi}_{\infty \to 2} \eps'$. Hence $g$ is not robust at $x$ up to an $\ell_2$ radius of $\eps= \eps' \norm{\Pi}_{\infty \to 2}$, as required.
\end{proof}

The following simple proposition shows how to obtain certified robustness guarantees in the representation $\calU$ e.g., DCT basis by performing appropriate training in another representation $\calX$ e.g., the co-ordinate or pixel basis. Let $O \in \R^{n \times n}$ represent an orthogonal matrix that represents the DCT transformation i.e., for an input $x \in \calX \subseteq \R^n$ let $\psi(x)=Ox$ represents its DCT representation (hence $\calU = \psi(\calX)$). 

\begin{proposition}\label{prop:rotation}
Suppose $\calX, \calU=\psi(\calX)$ be defined as above. 
For any classifier $g: \calU \to \calY$, consider the classifier $f: \calX \to \calY$ obtained as $f=g(Ox)$ where $O$ is an orthogonal matrix. 
For a point $x \in \calX$ if we denote by $r^{(\ell_2)}_f(x)$ the radius of robustness of $f$ at $x$ measured in $\ell_2$ norm, then we have that $r^{(\ell_2)}_g(O x ) = r^{(\ell_2)}_f(x)$. 
Moreover suppose the classifier $g(u)=g(\Pi u)~\forall u \in \calU$, for some projection matrix $\Pi$, then the robust accuracy of $g$ to $\ell_\infty$ perturbations in the representation $\calU$ satisfies $\accu_{\eps'}^{(\ell_\infty)}(g) \ge \accu_{\eps}^{(\ell_2)}(f) $ for any $\eps>0$ and $\eps'\le \eps/\norm{\Pi}_{\infty \to 2}$. 
\end{proposition}
\begin{proof}
We will prove by contradiction. Let $u, u' \in \psi(\calX)$ and let $u'=u+\delta_u$ be an adversarial perturbation of $u$ with $\norm{\delta_u}_2 = r$. Let $x=O^{-1} u$, and let $x'=O^{-1} u'$; note that $x, x'$ exist since $u, u' \in \psi(\calX)$. Moreover $\norm{x-x'}_2 = \norm{u-u'}_2$ since $O$ is a rotation (orthogonal matrix). Hence $x'$ is an adversarial example for $f$ at $x$, at a $\ell_2$ distance of $\delta_u$. This shows that $r^{(\ell_2)}_g(O x ) \ge r^{(\ell_2)}_f(x)$. A similar argument also shows that $r^{(\ell_2)}_g(O x ) \le r^{(\ell_2)}_f(x)$. The last part of the proposition follows by applying Proposition~\ref{prop:linftol2}.
\end{proof}

The following simple fact relates the $\infty \to 2$ operator norm and the $\ell_1$ sparsity of a projection matrix. This justifies the use of $\ell_1$ sparsity as an approximate relaxation or proxy for  $\norm{\Pi}_{\infty \to 2}$.
\begin{fact}\label{fact:sparsityconnection}
For any (orthogonal) projection matrix $\Pi \in \R^{n \times n}$ of rank $r$, we have 
$$\frac{\norm{\Pi}_1}{r} \le \norm{\Pi}_{\infty \to 2}^2 = \norm{\Pi}_{\infty \to 1} \le \norm{\Pi}_1 ,$$
where $\norm{\Pi}_1$ refers to the entry-wise $\ell_1$ norm of $\Pi$. 
\end{fact}
\begin{proof}
First we show the upper bound 
$$\norm{\Pi}_{\infty \to 1} = \max_{\substack{y,z \in \R^n: \norm{y}_\infty \le 1, \norm{z}_{\infty} \le 1}} \iprod{\Pi, yz^\top} \le \sum_{i,j} |\Pi(i,j)| = \norm{\Pi}_1.$$
We now show the lower bound. Let $\Pi=\sum_{i=1}^r v_i v_i^\top$, where $\set{v_i : i \in [r]}$ represents an orthonormal basis for the subspace given by $\Pi$. We have from the monotonicity of $\infty \to 1$ operator norm shown in \cite{ACCV} that
\begin{align*}
\norm{\Pi}_{\infty \to 1} &= \Bignorm{\sum_{i=1}^r v_i v_i^\top}_{\infty \to 1} \ge \max_{i \in [r]} \norm{v_i v_i^{\top}}_{\infty \to 1} = \max_{i \in [r]} \norm{v_i v_i^{\top}}_{1} \\
&\ge \frac{1}{r}\sum_{i \in [r]} \norm{v_i v_i^{\top}}_{1} \ge \frac{1}{r}\Bignorm{\sum_{i \in [r]} v_i v_i^{\top}}_{1}=\frac{\norm{\Pi}_1}{r},    
\end{align*}
where the last inequality uses the triangle inequality for matrix operator norms. 
 
\end{proof}

\section{Translating certified adversarial robustness from $\ell_2$ to $\ell_\infty$ perturbations} \label{sec:app-linf}
\subsection{Certifying the $\infty \to 2$ operator norm} \label{sec:mw-app}

\anote{Need to make it consistent with main body}

\anote{Add a note about run time comparison in appendix.}

We now describe our efficient algorithmic procedure that gives a certifies an upper bound on the $\infty \to 2$ norm of the given matrix. Moreover, as we will see in Theorem~\ref{thm:sdp-guarantee}, our algorithmic procedure comes with provable guarantees: it is guaranteed to output a value that is only a small constant factor off from the global optimum i.e., true value of the $\infty \to 2$ norm. 

For any matrix $A$, we first note that $\| A\|_{\infty \to 2}^2 = \|A^\top A\|_{\infty \to 1}$.  Moreover, by the variational characterization of operator norms, we have for $M=A A^\top \succeq 0$
\begin{equation} \label{eq:operator:problem}
\|A\|_{\infty \to 2}^2 = \max_{x: \|{x}\|_\infty \le 1}\max_{y: \|{y}\|_\infty \le 1} x^\top M y = \max_{x: \|{x}\|_\infty \le 1}x^\top M x.
\end{equation}

As stated in Section~\ref{sec:linf:methods}, our algorithm will apply to the more general problem \eqref{eq:QP:problem} where $M_{ii} \ge 0$ for all $i \in [n]$.
\begin{equation} \nonumber
\text{Given a symmetric matrix }M~\text{with } ~\forall i \in [n]:~ M_{ii} \ge 0 ,~\quad \max_{x: \|{x}\|_\infty \le 1}x^\top M x.
\end{equation}

Note that this is certainly satisfied by all $M \succeq 0$. 
We consider the following standard SDP relaxation  \eqref{eq:primal} for the problem, where $M \succeq 0$. The primal variables are represented by the positive semi-definite (PSD) matrix $X \in \R^{n \times n}$. The SDP dual of this relaxation, given in \eqref{eq:dual}, has variables $y_1, \dots, y_n \ge 0$ corresponding to the $n$ constraints \eqref{eq:primal:const1}. In what follows, for matrices $M,X$, $\iprod{M,X}:=\tr(M^\top X)$ represents the trace inner product. 

\fboxsep=0pt
\noindent\fbox{%
\vspace{5pt}
\noindent\begin{minipage}{.48\linewidth}
\begin{align}\label{eq:primal}
{\textbf{Primal SDP:}}~~~
\max_{X} & ~ \iprod{M,X} \\
\text{ s.t. } ~~~& X_{ii} \le 1, ~~~\forall i \in [n] \label{eq:primal:const1}\\
& X \succeq 0. \nonumber
\end{align}
\end{minipage}%
~~~\vline~~~
\begin{minipage}{.48\linewidth}
\begin{align}\label{eq:dual}
\textbf{Dual SDP:}~~~ 
\min_{y} & ~ \sum_{i \in [n]} y_i \\
\text{ s.t. } ~~~& diag(y) \succeq M  \label{eq:dual:const2}\\
& y \ge 0. \nonumber
\end{align}
\end{minipage}
~

\vspace{3pt}
}

Let us denote by $\sdpval$ the value of the optimal solution to the primal SDP \eqref{eq:primal}.
\anote{Move next line earlier?}
Recall that our algorithm goal is to output a value that is guaranteed to be a valid upper bound for $\norm{A}_{\infty \to 2}^2 = \norm{M}_{\infty \to 1}$.\footnote{Hence a fast algorithm that potentially finds a local optimum for the problem will not suffice for our purposes; we need an upper bound on the global optimum of \eqref{eq:operator:problem}.}  
The above primal SDP~\eqref{eq:primal} is a valid relaxation, and it is also tight up to a factor of $\pi/2$ i.e.,  $ (2/\pi) \sdpval \le \norm{M}_{\infty \to 1} \le \sdpval$. 
Moreover by weak duality, any feasible solution to the dual SDP~\eqref{eq:dual} gives a valid upper bound for the primal SDP value, and hence $\norm{M}_{\infty \to 1}$. However, the above SDP is computationally intensive to solve using off-the-shelf SDP solvers (even for CIFAR-10 images, $X$ is $1024 \times 1024$) \anote{Justify this better.}. Instead we design an algorithm based on the {\em multiplicative weight update} (MWU) framework for solving SDPs ~\cite{kleinLu,AroraHK}.

Our algorithm is described in Algorithm~\ref{algo:sdp} of Section~\ref{sec:linf:methods}. 
Our algorithm differs from the standard MWU approach (and analysis)~\cite{Kalethesis} treats the constraints \eqref{eq:primal:const1} and the objective~\eqref{eq:primal} similarly. Firstly, by treating the objective differently from the constraints, we ensure that the width of the convex program is smaller; in particular, the ``width'' parameter does not depend on the value of the objective function anymore. This allows us to get a better convergence guarantee. Secondly, we always maintain a dual feasible solution that gives a valid upper bound on the objective value (this also allows us to do an early stop if appropriate).
We remark that for {\em every} choice of $\alpha$ (and the update rule), $\cub=n\lambda$ gives a valid upper bound (due to dual feasibility). This is encapsulated in the following simple proposition 
where $\sdpval$ refers to the optimal solution value of the \eqref{eq:primal}.    

\begin{proposition} \label{prop:dual}
For any $\alpha \in \R_{\ge 0}^n $ with $\sum_{i=1}^n \alpha(i) =n$, if $\lambda$ is the maximum eigenvalue 
$$\lambda = \lambda_{\max}\Big( (\diag(\alpha)^{-1/2} M \diag(\alpha)^{-1/2}) \Big), \text{ then we have } \sdpval \le n \lambda,$$
and $y=\lambda \alpha$ is feasible for the dual SDP \eqref{eq:dual} and attains an objective value of $n \lambda$. 
\end{proposition}
\begin{proof}
Consider $y= \lambda \alpha$. Firstly, the dual objective value at $y$ is $\sum_{i=1}^n y_i = \lambda \sum_i \alpha(i) = n\lambda$, as required. Moreover $y \ge 0$, since the $\alpha \ge 0$. Finally, \eqref{eq:dual:const2} is satisfied since by definition of $\lambda$, 
\begin{align*}
    \diag(\alpha)^{-1/2} M \diag(\alpha)^{-1/2} \preceq \lambda \cdot I ~~\implies~~ M \preceq \lambda \cdot \diag(\alpha) = \diag(y),  
\end{align*}
by pre-multiplying and post-multiplying by $\diag(\alpha)^{1/2}$. Finally, from weak duality $\sdpval \le n \lambda$. 
\end{proof}

\paragraph{Analysis of the algorithm}

We show the following guarantee for our algorithm. 

\begin{theorem}\label{thm:sdp-guarantee}
For any $\delta>0$, any symmetric matrix $M$ with $M_{ii} \ge 0~\forall i \in [n]$ , Algorithm~\ref{algo:sdp} on input $M$, with parameters $\delta$ and $\rho=O(n/\delta)$ after $T=O(n \log n/\delta^3)$ iterations finds a solution $\widehat{X} \in \R^{n \times n}$ and $\widehat{y} \in \R^n$ such that 

\vspace{-\topsep}
\begin{enumerate}
\setlength{\parskip}{0pt}
  \setlength{\itemsep}{0pt plus 1pt}
    \item[(a)] $\widehat{y}$ is feasible for the dual SDP \eqref{eq:dual} such that $\iprod{M,\widehat{X}} = \sum_{i=1}^n y_i$.  
    \item[(b)] $\widehat{X} \succeq 0$, $\widehat{X}_{ii} \le 1+\delta$ for all $i \in [n]$, and $\iprod{M, \widehat{X}}\ge \sdpval$, the primal SDP value.
\end{enumerate}
\vspace{-\topsep}
\end{theorem}
The above Theorem~\ref{thm:sdp-guarantee} and Proposition~\ref{prop:dual} together imply Theorem~\ref{thm:sdp-combined}. 
We remark that the $\tilde{O}(n \log n) \cdot T_{eig}$ running time guarantee almost matches the guarantees for Klein and Lu for Max-Cut SDP~\cite{kleinLu} (up to a $O(1/\delta)$ factor), but our guarantees apply to the SDP for the more general Quadratic Programming problem. Note that the maximum eigenvalue can be computed within $\delta$ accuracy in $O(m/\delta)$ time where $m$ is the number of non-zeros of $M$, (see e.g., ~\cite{kleinLu} for a proof and use in MWU framework). To the best of our knowledge, the best known algorithm prior to our work for solving the SDP to this general Quadratic Programming problem (or even problem~\eqref{eq:operator:problem}) is by Arora, Hazan and Kale~\cite{AroraHK, Kalethesis} who give a $O(\frac{n^{1.5}}{\delta^{2.5}} \cdot \min\set{m,n^{1.5}/(\delta \alpha^*)})$, where the optimal solution value is $\sdpval=\alpha^* \norm{M}_1$. Compared to our algorithm's running time of $\tilde{O}(nm\log n)$, even when $\alpha^*=\Omega(1)$ the previous best requires a running time of $\tilde{O}(n^{1.5} \min\set{m,n^{1.5})})$; but it can even be the case that $\alpha^*=O(1/n)$~\cite{charikarwirth}.  Finally recall that upper bound given by the SDP is only off by a factor of $\pi/2$ for PSD matrices, that is the best possible assuming $P \ne NP$ 
(for general matrices the approximation factor could be $O(\log n)$) \cite{nesterov1998semidefinite, charikarwirth, alon2006quadratic, bhattiprolu2018inapproximability}.  

The analysis closely mirrors the analysis of the Multiplicative Weights algorithm for solving SDPs due to Klein and Lu~\cite{kleinLu}, and Arora, Hazan and Kale ~\cite{AroraHK,Kalethesis}. The main parameter that affects the running time of the multiplicative weights update method is the width parameter $\rho'$ of the SDP constraints. The analysis of Klein and Lu ~\cite{kleinLu} is specialized for the Max-Cut problem (which is a special case of \eqref{eq:operator:problem} where $M$ is a graph Laplacian), where they achieve a bound of $O(n\log n/\delta^2)$ eigenvalue computations. However their analysis does not directly extend to our more general setting (they crucially use the fact that the optimum solution value is at least $\Omega(\tr(M))$ to get a small bound on the width).  In the more general framework of \cite{AroraHK}, the optimization problem is first converted into a feasibility problem (hence the objective also becomes a constraint). For SDPs of the form \eqref{eq:primal}, the width parameter could be reasonably large and could result in $\Omega(n^{3/2})$ iterations (see Section 6.3 of \cite{Kalethesis}). In our analysis, we ensure that the width is small by treating the objective separately from the constraints. A crucial step of our algorithm is introducing an amount $\delta$ to each of the weights which solving the eigenvalue maximization in each step. This correction term of $\delta$ is important in getting an upper bound on the violation of each constraint, and hence the width of the program.  Finally, we present a slightly different analysis using the SDP dual, that has the additional benefit of providing a certificate of optimality. 

\paragraph{Proof of Theorem~\ref{thm:sdp-guarantee}.}

Let at step $t$ of the iteration, $\alpha^{(t)}$ be the weight vector, $X^{(t)}$ be the candidate SDP solution and  $y^{(t)}$ be the candidate dual solution maintained by the Algorithm~\ref{algo:sdp} in step \ref{eq:iterate:def}. Also $\tilde{\alpha}^{(t)}:= (1-\delta) \alpha^{(t)}+ \delta \bfone$, where $\bfone=(1,1,\dots,1)$. It is easy to see that
\begin{align*}
  \forall t \le T,~~  X^{(t)} &= \frac{1}{t}\sum_{\ell=1}^t v^{(\ell)} (v^{(\ell)})^\top, ~~~ \text{and  }~\quad~  y^{(t)} = \frac{1}{t}\sum_{\ell=1}^t \lambda^{(\ell)} \tilde{\alpha}^{(\ell)}.
\end{align*}

We first establish the following lemma, which immediately implies part (a) of the theorem statement. 
\begin{lemma}\label{lem:dualproperties}
For all $t \le T$, $y^{(t)}$ is feasible for the dual SDP~\eqref{eq:dual}, and $\iprod{M,X^{(t)}}=\sum_{i} y^{(t)}_i$. 
\end{lemma}
\begin{proof}
We first check feasibility of the dual solution. First note that in every iteration $t \in [T]$, $\tilde{\alpha}^{(t)} \ge 0$ and $\lambda^{(t)} \ge 0$. The latter is because $\tr(\diag(\tilde{\alpha}^{(t)})^{-1/2} M \diag(\tilde{\alpha}^{(t)})^{-1/2}) \ge 0$ as all the diagonal entries of $M$ are non-negative, and hence $\lambda_{\max}(\diag(\tilde{\alpha}^{(t)})^{-1/2} M \diag(\tilde{\alpha}^{(t)})^{-1/2}) \ge 0 $. 
Hence $y^{(t)} \ge 0$. Moreover, by definition of $\lambda^{(t)}$
\begin{align*}
     M &\preceq \lambda^{(\ell)} \cdot \diag(\tilde{\alpha}^{(\ell)}) ~~\forall \ell \le T, ~~(\text{ since } \diag(\tilde{\alpha}^{(t)})^{-1/2} M \diag(\tilde{\alpha}^{(t)})^{-1/2} \preceq \lambda^{(t)} I) ,\\
    M&\preceq \frac{1}{t} \sum_{\ell=1}^t \lambda^{(\ell)} \cdot \diag(\tilde{\alpha}^{(\ell)}) = \frac{1}{t} \sum_{\ell=1}^t \diag(y^{(\ell)}).     
\end{align*}

Finally, we establish $\iprod{M,X^{(t)}}=\sum_{i} y^{(t)}_i$. Note that by definition $v^{(\ell)}= \sqrt{n} \diag(\tilde{\alpha}^{(\ell)})^{-1/2} u^{(\ell)}$, where $u^{(\ell)}$ is the top eigenvector of $(\diag(\tilde{\alpha})^{-1/2} M \diag(\tilde{\alpha})^{-1/2})$, and $\lambda^{(\ell)}$ is its eigenvalue. Hence
\begin{align*}
    \iprod{M,X^{(t)}} &= \frac{1}{t} \sum_{\ell=1}^t \iprod{ M, v^{(\ell)} (v^{(\ell)})^{\top} }= \frac{1}{t} \sum_{\ell=1}^t (v^{(\ell)})^{\top} M v^{(\ell)} \\
    &= n \frac{1}{t} \sum_{\ell=1}^t \Big((u^{(\ell)})^{\top} \diag(\tilde{\alpha}^{(\ell)})^{-1/2} M \diag(\tilde{\alpha}^{(\ell)})^{-1/2} u^{(\ell)}\Big) =  \frac{1}{t} \sum_{\ell=1}^t \lambda^{(\ell)} \cdot n\\
    \text{On the other hand, }~ \sum_{i=1}^n y^{(t)}_i &= 
    \sum_{i=1}^n \frac{1}{t} \sum_{\ell=1}^t \lambda^{(\ell)} \tilde{\alpha}^{(\ell)}(i) = \frac{1}{t} \sum_{\ell=1}^t \lambda^{(\ell)} \cdot \sum_{i=1}^n \tilde{\alpha}^{(\ell)}(i) =  \frac{1}{t} \sum_{\ell=1}^t \lambda^{(\ell)} \cdot n.
\end{align*}
This proves the lemma. 
\end{proof}

We now complete the proof of Theorem~\ref{thm:sdp-guarantee}. 
\begin{proof}
Recall that $\widehat{X}=X^{(T)}, \widehat{y}=y^{(T)}$.  Consider the SDP solution $X' = \tfrac{1}{1+\delta} \widehat{X}$. From Lemma~\ref{lem:dualproperties}, we see that $y^{(T)}$ is feasible and 
\begin{align*}
    \iprod{M,X'}&= \frac{1}{1+\delta} \iprod{M, \widehat{X}} =  \frac{1}{1+\delta} \sum_{i=1}^n \widehat{y}(i) \ge \frac{1}{1+\delta}\cdot \sdpval,
\end{align*}
where the last inequality follows from weak duality of the SDP. Hence it suffices to show that $X'$ is feasible i.e., $\widehat{X}_{ii} \le 1+\delta$ for all $i \in [n]$.  

We prove $\max_{i \in [n]} \widehat{X}_{ii} \le 1+\delta$ by following the same multiplicative weight analysis in \cite{Kalethesis} (see Section 2, Theorem 5), but only restricted to the $n$ constraints of the form $X_{ii} \le 1$.

There is one expert for each $i \in [n]$ corresponding to the constraint $X_{ii} \le 1$. At each iteration $t$, we consider a probability distribution $p^{(t)}=\tfrac{1}{n} \alpha^{(t)}$. The loss of expert $i$ at time $t$ is given by $m^{(t)}(i)=\tfrac{1}{\rho}(1-(v^{(t)}_i)^2)$ for $\rho=n/\delta$ . We note that in each iteration $\ell$, 
\begin{align*}
\sum_i (\alpha^{(t)}_i + \delta) (v^{(t)}_i)^2 \le n &~\implies~~ \forall i \in [n],~~ (v^{(t)}_i)^2 \le \frac{n}{\delta} \le \rho\\ 
\text{Hence }~ \frac{1}{\rho} \ge m_i^{(t)} = \frac{1}{\rho}\Big(1-(v^{(t)}_i)^2\Big) &\ge -\frac{(n/\delta -1)}{\rho} \ge -1 .    
\end{align*}
 
From the guarantees of the multiplicative weight update method (see Section 2, Theorem 2 of \cite{Kalethesis}), we have for each $i \in [n]$
\begin{align}
    \sum_{t=1}^T \sum_{i=1}^n m_i^{(t)} p_i^{(t)} &\le \sum_{t=1}^T m_i^{(t)} + \delta \sum_{t=1}^T |m_i^{(t)}| + \frac{\ln n}{\delta} \nonumber\\
    &=(1-\delta)\sum_{t=1}^T m_i^{(t)} + 2\delta \sum_{\substack{t\in [T]\\ m_i^{(t)}>0}} |m_i^{(t)}| + \frac{\ln n}{\delta} \nonumber\\
        &\le \frac{(1+\delta)}{\rho} \sum_{t=1}^T \big(1-(v^{(t)}_i)^2 \big) + \frac{2\delta}{\rho} + \frac{\ln n}{\delta} \nonumber\\
    \frac{\rho}{T}\sum_{t=1}^T \sum_{i=1}^n m_i^{(t)} p_i^{(t)} &\le (1-\delta) \cdot (1-\widehat{X}_{ii}) + 2 \delta + \frac{\rho \ln n}{T \delta}. \label{eq:mw:1}
\end{align}
On the other hand, using the fact that $\sum_i \alpha^{(t)}_i=n$, and $\sum_i \tilde{\alpha}^{(t)}_i (v^{(t)}_i)^2 = n$ we have  
\begin{align}
    \frac{\rho}{T}\sum_{t=1}^T \sum_{i=1}^n m_i^{(t)} p_i^{(t)} &= \frac{1}{T}\sum_{t=1}^T \frac{1}{n }\sum_{i=1}^n \big(1- (v^{(t)}_i)^2  \big)\alpha_i^{(t)} \nonumber\\
    &= \frac{1}{T}\sum_{t=1}^T \sum_{i=1}^n \alpha_i^{(t)} - \frac{1}{T}\sum_{t=1}^T \frac{1}{n}\sum_{i=1}^n \alpha_i^{(t)} (v^{(t)}_i)^2  = -\frac{1}{T}\sum_{t=1}^T \frac{1}{n}\sum_{i=1}^n \Big(\frac{\tilde{\alpha_i}^{(t)}}{1-\delta} + \frac{\delta}{1-\delta} \Big) (v^{(t)}_i)^2  + 1 \nonumber\\ 
    &= -\frac{1}{T}\sum_{t=1}^T \frac{1}{(1-\delta)n}\sum_{i=1}^n \tilde{\alpha_i}^{(t)} (v^{(t)}_i)^2 + \frac{1}{T}\sum_{t=1}^T \frac{1}{(1-\delta)n}\sum_{i=1}^n \delta (v^{(t)}_i)^2)   + 1 \nonumber\\
    &= -\frac{1}{1-\delta} + 1 +  \frac{1}{T}\sum_{t=1}^T \frac{1}{(1-\delta)n} \sum_{i=1}^n \delta (v^{(t)}_i)^2 \nonumber\\
    &\ge -\frac{\delta}{1-\delta}  \label{eq:mw:2}
\end{align}
Combining \eqref{eq:mw:1} and \eqref{eq:mw:2}, we get that for $\delta \in (0,1/2)$, and $T=\rho \log n/ ((1-\delta)\delta^2) = O(n \log n/\delta^3)$ 
\begin{align*}
\forall i \in [n],~~ (1-\widehat{X}_{ii}) &\ge -\frac{\delta}{(1-\delta)^2}- \frac{2 \delta}{1-\delta} - \frac{\rho \ln n}{T \delta(1-\delta)} \\    
\widehat{X}_{ii} &\le 1+ \delta (1+4\delta) +  2\delta(1+2\delta)+\delta \le 1+ 8\delta. \end{align*} 
This completes the proof of part (b). 
It is also straightforward to see that in above analysis, a $(1+\delta)$ approximate eigenvalue method can also be used to get a similar guarantee with an extra $(1+\delta)$ factor loss in the objective value. 
\end{proof}

We remark that the larger dependence of $\rho=O(n/\delta)$ is needed to ensure that in  each iteration $(v_i^{(t)})^2 / \rho \le 1$. If this condition is satisfied for $\rho \ll n/\delta$, then the iteration bound is $O(\rho \log n)/\delta^2$. This also justifies the use of a more aggressive i.e., smaller choice of $\rho$ in practice.  


\ifNeurIPS
\subsection{Finding Robust Low-rank Representations} \label{sec:robustprojheuristic}

We now show how to find a good low-rank robust projections when it exists in the given representation. Given a dataset $A$, our goal is to find a (low-rank) projection $\Pi$ that gets small reconstruction error $\norm{A - \Pi A}_F^2$, while ensuring that $\norm{\Pi}_{\infty \to 2}$ is small. Awasthi et al.~\cite{ACCV} formulate this as an optimization problem that is NP-hard, but show polynomial time algorithms based on the Ellipsoid algorithm that gives constant factor approximations. However, the algorithm is impractical in practice because of the Ellipsoid algorithm, and the separation oracle used by it (that in turn involves solving an SDP). 
We instead use the connection to sparsity described in Section~\ref{sec:linf:methods} to devise a much faster heuristic based on sparse PCA to find a good projection $\Pi$ (see \eqref{fact:sparsityconnection} in the appendix for a formal justification). We just use an off-the-shelf heuristic for sparse PCA (the scikit-learn sparse PCA implementation~\cite{scikit-learn} based on ~\cite{mairal2009sparsePCA}), along with our certification procedure Algorithm~\ref{algo:sdp}). 

\begin{algorithm}[H]
\caption{Find a Robust Projection}
\label{algo:spca}
\begin{algorithmic}[1]
\Function{RobustProjection}{data $A \in R^{m \times n}$, rank $k$, reconstruction error $\delta$} 
\State Set $M:=(A^\top A)/\tr(A^\top A)$. Initialize $\widehat{\Pi} \gets \emptyset, \widehat{\kappa}=\infty$. 
\For{different values of $r \le k$}
\State Find $r$-PCA of the $M$ to get a rank $r$ projection $\Pi_1$. $M' \gets M - \Pi_1 M \Pi_1$  
\State Run sparse PCA on $M'$ to find a rank $(k-r)$ projection $\Pi_2$. Set $\Pi=\Pi_1 + \Pi_2$. 
\State Run $\textsc{CertifySDP}(\Pi,\delta=1/4,\rho)$ to get an upper bound $\kappa$.
\If{$\kappa < \widehat{\kappa}$ and if $\iprod{M, I-\Pi} \le \delta$}
    \State $\widehat{\Pi} \gets \Pi$, $\widehat{\kappa}=\kappa$.
\EndIf
\EndFor
\State Output $\widehat{\Pi}$, $\widehat{\kappa}$. 
\EndFunction
\end{algorithmic}
\end{algorithm}
\fi

\section{Imperceptibility in the DCT basis and Training Certified $\ell_\infty$ Robust Networks}
\label{sec:app-imperceptibility}

\paragraph{Adversarial Examples for CIFAR-10 images. }
 We take a ResNet-32 network that has been trained on the CIFAR-10 datasets via the PGD based method of \cite{madry2017towards} for robustness to $\ell_\infty$ perturbations. We then generate imperceptible adversarial examples for the test images via projected gradient ascent and an $\ell_\infty$ perturbation radius of $\epsilon = 0.09$. See Figure~\ref{fig:adv_examples}  for a few of the original images and the corresponding adversarial perturbations.
\begin{figure}
    \centering
    \includegraphics[width=0.8\textwidth]{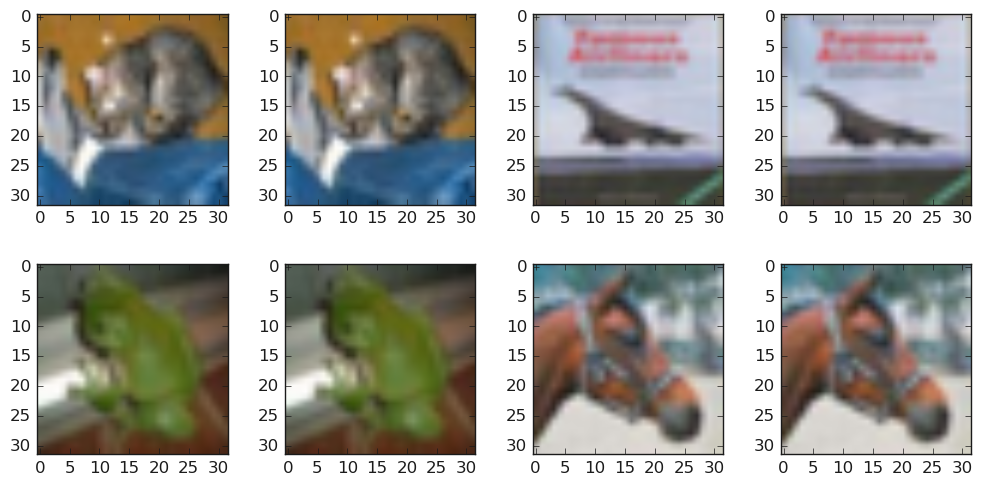}
    \caption{The images on the left correspond to the original images and the images on the right correspond to imperceptible adversarial examples within an $\ell_\infty$ radius of $\eps \le 0.09$. \label{fig:adv_examples}}
\end{figure}

\section{Experimental Evaluation of  MWU-based SDP Certification Algorithm~\ref{algo:sdp}}
\label{app:sdp-runtme}
\begin{figure}[h]
\centering
\includegraphics[width=1\textwidth]{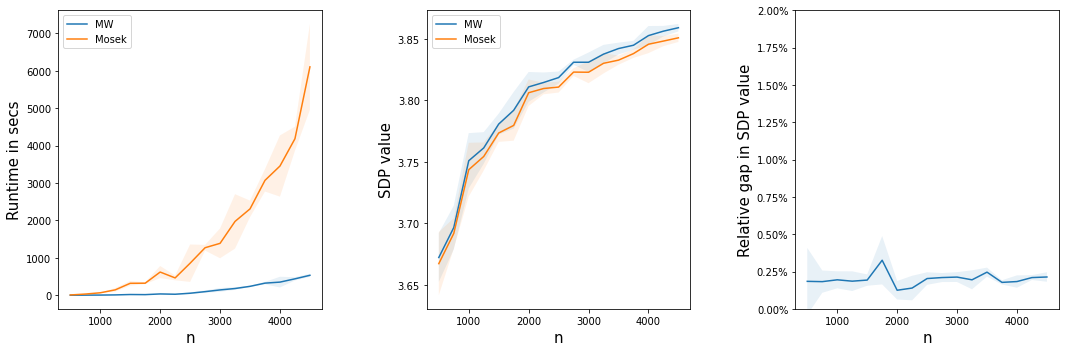}
\caption{Comparison of the running time of Algorithm~\ref{algo:sdp}(left plot) with the MOSEK solver for PSD matrices of varying sizes~($n$ from $500$ to $4500$). The middle plot shows the SDP values output by the two procedures. The right plot shows the relative error in the SDP value output by our procedure as compared to the value output by the MOSEK solver. 
\label{fig:mw_vs_mosek_psd} }
\end{figure}

In this section we empirically evaluate the effectiveness of our multiplicative weights based procedure from Algorithm~\ref{algo:sdp} for solving the general Quadratic Programming~(QP) problem as defined in \eqref{eq:QP:problem}. Recall that this corresponds to the following optimization problem:
\begin{equation} \label{eq:QP:problem-app}
\text{Given a symmetric matrix }M~\text{with } ~\forall i \in [n]:~ M_{ii} \ge 0 ,~\quad \max_{x: \|{x}\|_\infty \le 1}x^\top M x.
\end{equation}
Recall that in our Algorithm~\ref{algo:sdp} every iterations involves a single maximum eigenvalue computation, for which we use an off the shelf subroutine from Python's scipy package. In our implementation of Algorithm~\ref{algo:sdp} we add an early stopping condition if the dual value does not improve noticeably in successive rounds. Recall that Proposition~\ref{prop:dual} proves that this also gives a valid upper bound on the primal SDP value.

We consider two scenarios, one where $M$ is a positive semi-definite matrix~(PSD) and the other when $M$ is a symmetric matrix with non-negative diagonal entries. In each case we compare the running time of our algorithm and the dual SDP value that it outputs with the corresponding values obtained by using a state of the art SDP solver on the same instance. As a comparison for our experiments we choose the SDP solver from the commercial optimization software MOSEK~\cite{mosek}.  
For consistency, the comparison experiments were run on a single core of a Microsoft Surface Pro 3 tablet/computer with Intel Core i7-8650U CPU \@ 1.90 GHz with 16GB RAM. 

For the case when $M$ is a PSD matrix, we generate a random $n \times n$ matrix $A$ that contains entries drawn from a standard normal distribution and set $M = A A^\top$, and renormalize the matrix so that the trace is $1$. We vary $n$ from $500$ to $4500$ in increments of $250$ (we used $5$ random trials for each value of $n$). Figure~\ref{fig:mw_vs_mosek_psd} shows the comparison of our algorithm with the SDP solver from MOSEK for different values of $n$, averaged over the trials along with error bars. Notice that the running time of our algorithm is an order of magnitude faster than the MOSEK solver particularly for larger values of $n$, and furthermore the SDP values output are within $0.5\%$ of the values output by exactly optimizing the SDP via the MOSEK solver. We also include a table of the average run times for a few different values of $n$ in Table~\ref{tbl:sdp-comparison-runtime}. We remark that other non-commercial SDP solvers like CVXPY~\cite{cvxpy} timed out even for for $n>800$. 

A similar trend holds in Figure~\ref{fig:mw_vs_mosek_qp} where we consider randomly generated instances of the more general QP problem \eqref{eq:QP:problem-app}. Here $M$ is chosen to be a random symmetric $n \times n$ matrix with entries drawn from the standard normal distribution and replace the diagonal entries of $M$ with their absolute values (so that the diagonals are non-negative). 
\begin{table}
\centering
\begin{tabular}{ |c|c|c| }
\hline
n & Running time (in s) of our Algorithm~\ref{algo:sdp} & Running time (in s) of MOSEK \\
\hline
$500$ & $1.384 \pm 0.324$ & $8.444 \pm 0.389$\\
\hline
$1000$ & $5.631 \pm 1.190$ & $66.578 \pm 5.432$\\
\hline
$2000$ & $35.747 \pm 8.960$ & $620.615 \pm 161.96$\\
\hline
$3000$ & $143.07 \pm 17.29$& $1387.69 \pm 398.26$\\
\hline
$4000$ & $351.81 \pm 142.97 $& $3453.95 \pm 818.50$\\
\hline
$4500$ & $518.56$ $\pm$ $27.81$ & $5579.93 \pm 523.53$\\
\hline
\end{tabular}
\vspace{4pt}
 \caption{The average time~(in seconds) taken by Algorithm~\ref{algo:sdp} and the MOSEK solver on random PSD matrices for a few different sizes ($n$). The mean and standard deviation are computed over $5$ independent runs for each value of $n$. See also Figure~\ref{fig:mw_vs_mosek_psd} for the plot based on more values of $n$ taken in increments of $250$. 
 \label{tbl:sdp-comparison-runtime}}
\end{table}

Furthermore, unlike MOSEK, our procedure can scale to much larger values of $n$. As an example, Table~\ref{tbl:sdp-runtime} shows the running time needed for our procedure to perform $200$ iterations on PSD matrices $M$ of sizes ranging from $5000$ to $2000$. This demonstrates the scalability of our method for larger datasets with higher dimensions. For scalability purposes, the experiments below were run on a machine with access to a single GPU. Hence, the runtimes reported in Table~\ref{tbl:sdp-runtime} are not directly comparable to those in Table~\ref{tbl:sdp-comparison-runtime}.
\begin{figure}[h]
\centering
\includegraphics[width=1\textwidth]{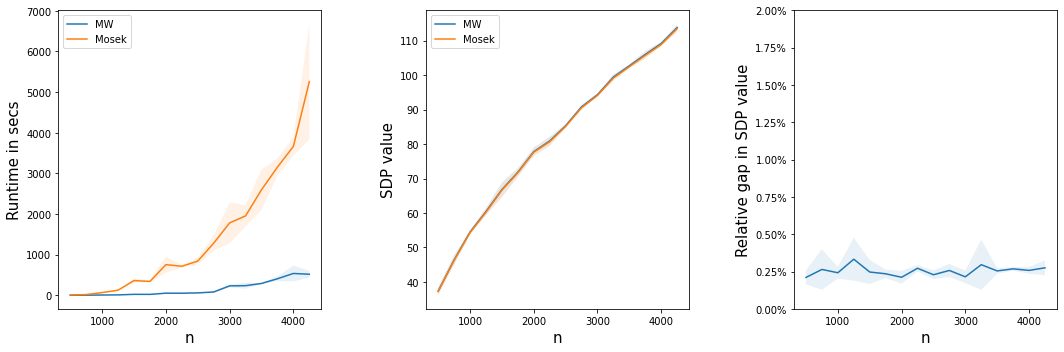}
\caption{Comparison of the running time of Algorithm~\ref{algo:sdp}(left plot) with the MOSEK solver for general symmetric matrices with non-negative diagonal entries of varying sizes~($n$). The middle plot shows the SDP values output by the two procedures. The right plot shows the relative error in the SDP value output by our procedure as compared to the value output by the MOSEK solver. \label{fig:mw_vs_mosek_qp} }
\end{figure}

\begin{table}[H]
\centering
\begin{tabular}{ |c|c| }
\hline
n & Running time (in s) of our Algorithm~\ref{algo:sdp}\\
\hline
$5000$ & $104.92 \pm 4.04$\\
\hline
$6000$ & $156.29 \pm 4.99$\\
\hline
$7000$ & $213.80 \pm 7.86$\\
\hline
$8000$ & $272.83 \pm 17.29$\\
\hline
$9000$ & $346 \pm 8.99$\\
\hline
$10000$ & $378.70$ $\pm$ $17.4$\\
\hline
$15000$ & $1092.14$ $\pm$ $103.79$\\
\hline
$20000$ & $1801.08$ $\pm$ $63.95$\\
\hline

\end{tabular}
\vspace{4pt}

 \caption{The time~(in seconds) taken by Algorithm~\ref{algo:sdp} to perform $200$ iterations on random PSD matrices of varying sizes ($n$). The mean and standard deviation are computed over $5$ independent runs for each value of $n$.
 \label{tbl:sdp-runtime}}
 \vspace{-10pt}
\end{table}

We remark that our algorithm is specific to the Quadratic Programming SDP (which is a large class of problems in itself), while MOSEK SDP solver is a general purpose commercial SDP solver based on interior-point methods that is more accurate.      
 Our MWU-based algorithm (Algorithm~\ref{algo:sdp}) for the Quadratic Programming SDP~\eqref{eq:primal} gives much faster algorithms for approximately solving the SDP (along with dual certificate of the upper bounds), while not compromising much on approximation loss in the value. Hence the experiments are consistent with the running time improvements suggested by the theoretical results in Section~\ref{sec:mw}.


\section{Robust Projections for Audio Data in DCT basis}
\label{app:audio}

We now perform an empirical evaluation on audio data to demonstrate the existence of good low-dimensional robust representations in the DCT basis. We consider the Mozilla CommonVoice speech-to-text audio dataset (english en\_1488h\_2019-12-10)\footnote{\href{https://voice.mozilla.org/en/datasets}{{https://voice.mozilla.org/en/datasets}}} that was also considered by the work of Carlini et al.~\cite{carlini2018audio} on audio adversarial examples. We use a standard approach (that is also employed by \cite{carlini2018audio}) to convert each audio file that may be of variable duration into a representation in high dimensional space. Each
element $x_i$ is a signed 16-bit value. First each audio file is converted into a sequence of overlapping frames corresponding to time windows; and each frame is represented as an $n$-dimensional vector of $16$ bit values where $n$ is given by the product of the window size and the sampling rate (for the Commonvoice dataset $n=1200$). Hence each audio file corresponds to a sequence of frames, each of which is represented by an $n$ dimensional vector.




\begin{table}[h]
    \centering
\begin{tabular}{ |c|c|c|c| }
\hline
Trial number & Number of frames N & $\infty \to 2$ norm & Projection error $\%$\\
\hline
1 & $421469$ & $21.285$ & $4.974$\\
\hline
2 & $ 480773$ & $22.899$ & $5.056$\\
\hline
3 & $447353$ & $21.962$ & $5.578$\\
\hline
4 & $461034$ & $24.439$ & $5.247$\\
\hline
5 & $438954$ & $23.840$ & $5.392$\\
\hline
6 & $452581$ & $22.902$ & $5.025$\\
\hline
\end{tabular}
\vspace{4pt}
 \caption{\label{tbl:audio}The table shows SDP upper bounds on $\infty \to 2$ norm for projection matrices of rank $r=200$ obtained by Algorithm~\ref{algo:spca} for $6$ trials each with random $1000$ audio samples from the Mozilla Common Voice dataset. Each random sample of $1000$ audio samples corresponds to roughly $N \approx 10000$ frames each of $n=1200$ dimension in the DCT basis.}
 
\end{table}

For our experiments, in each random trial we consider $1000$ randomly chosen audio files, and consider the data matrix (with $n$-dimensional columns) consisting of all the frames corresponding to these audio files (and we consider $6$ such random trials).  We first use the sparse PCA based heuristic (Algorithm~\ref{algo:spca} in Appendix~\ref{sec:mw-app}) to find a projection matrix of rank $r=200$. We then use Algorithm~\ref{algo:sdp} to compute upper bounds on the ${\infty \to 2}$ operator norm of the projections matrices. Table~\ref{tbl:audio} shows the values of the operator norms certified by our algorithm for the projection matrices that are found, along with the reconstruction/ projection error, expressed as a percentage. Notice the obtained subspaces have operator norm values significantly smaller than $\sqrt{n} \approx 34.641$.

\end{document}